\documentclass{article}

\usepackage[final,nonatbib]{neurips_2021}

\usepackage[utf8]{inputenc} % allow utf-8 input
\usepackage[T1]{fontenc}    % use 8-bit T1 fonts
\usepackage{hyperref}       % hyperlinks
\usepackage{url}            % simple URL typesetting
\usepackage{booktabs}       % professional-quality tables
\usepackage{amsfonts}       % blackboard math symbols
\usepackage{nicefrac}       % compact symbols for 1/2, etc.
\usepackage{microtype}      % microtypography
\usepackage{xcolor}         % colors
\usepackage{graphicx}
\usepackage{placeins}
\usepackage{wrapfig}
\usepackage{tabularx}
\usepackage{subcaption}

\usepackage{amsmath}
\usepackage{amssymb}

\DeclareMathOperator*{\argmin}{arg\,min}
\DeclareMathOperator{\proj}{proj}

\usepackage{mleftright}
\usepackage{hyperref}
\hypersetup{
    colorlinks=true,
    citecolor=blue,
    linkcolor=blue,
    filecolor=magenta,      
    urlcolor=cyan,
}
\usepackage{subcaption}

\usepackage{amsthm}
\usepackage{cleveref}
\theoremstyle{plain}

\newtheorem{proposition}{Proposition}[section]
\newtheorem{manualpropositioninner}{Proposition}
\newenvironment{manualproposition}[1]{%
	\renewcommand\themanualpropositioninner{#1}%
	\manualpropositioninner
}{\endmanualpropositioninner}
\theoremstyle{definition}
\newtheorem{assumption}{Assumption}[section]
\newtheorem{definition}{Definition}[section]

\theoremstyle{remark}
\newtheorem{remark}{Remark}[section]

\usepackage[ruled,vlined]{algorithm2e}

\title{Explaining Latent Representations with a Corpus of Examples}

\author{%
  Jonathan Crabbé \\
  University of Cambridge\\
  \texttt{jc2133@cam.ac.uk} \\
  \And
  Zhaozhi Qian \\
  University of Cambridge \\
  \texttt{zq224@maths.cam.ac.uk} \\
  \AND
  Fergus Imrie \\
  UCLA\\
  \texttt{imrie@g.ucla.edu} \\
  \And
  Mihaela van der Schaar \\
  University of Cambridge \\
  The Alan Turing Institute \\
  UCLA \\
  \texttt{mv472@cam.ac.uk} \\
}

\begin{document}

\maketitle

\newcommand{\R}{\mathbb{R}}
\renewcommand{\H}{\mathcal{H}}
\newcommand{\X}{\mathcal{X}}
\newcommand{\Y}{\mathcal{Y}}
\newcommand{\D}{\mathcal{D}}
\newcommand{\Dtrain}{\mathcal{D}_{\text{train}}}
\newcommand{\Dtest}{\mathcal{D}_{\text{test}}}
\newcommand{\Dout}{\mathcal{D}_{\text{out}}}
\newcommand{\Duk}{\mathcal{D}_{\text{UK}}}
\newcommand{\Dusa}{\mathcal{D}_{\text{USA}}}
\newcommand{\Dmnist}{\mathcal{D}_{\text{MNIST}}}
\newcommand{\Demnist}{\mathcal{D}_{\text{EMNIST}}}
\newcommand{\Cusa}{\mathcal{C}_{\text{USA}}}
\newcommand{\Cuk}{\mathcal{C}_{\text{UK}}}
\newcommand{\CH}{\mathcal{CH}}
\newcommand{\C}{\mathcal{C}}
\newcommand{\T}{\mathcal{T}}
\newcommand{\N}{\mathbb{N}}
\newcommand{\f}{\textbf{f}}
\newcommand{\w}{\textbf{w}}
\newcommand{\p}{\textbf{p}}
\newcommand{\g}{\textbf{g}}
\renewcommand{\l}{\textbf{l}}
\renewcommand{\j}{\textbf{j}}
\newcommand{\IG}{\text{IG}}
\newcommand{\x}{\textbf{x}}
\newcommand{\h}{\textbf{h}}
\newcommand{\z}{\textbf{z}}
\renewcommand{\P}{\textbf{P}}
\newcommand{\y}{\textbf{y}}
\newcommand{\A}{\textbf{A}}
\newcommand{\linec}{\boldsymbol{\gamma}^c}
\newcommand{\curvec}{\g \circ \linec}
\newcommand{\norm}[2]{\parallel #1 \parallel_{#2}}
\newcommand{\partderiv}[2]{\frac{\partial #1}{\partial #2}}
\newcommand{\deriv}[2]{\frac{d #1}{d #2}}
\maketitle

\begin{abstract}

Modern machine learning models are complicated. Most of them rely on convoluted latent representations of their input to issue a prediction. To achieve greater transparency than a black-box that connects inputs to predictions, it is necessary to gain a deeper understanding of these latent representations. To that aim, we propose SimplEx: a user-centred method that provides example-based explanations with reference to a freely selected set of examples, called the corpus. SimplEx uses the corpus to improve the user’s understanding of the latent space with post-hoc explanations answering two questions: (1) Which corpus examples explain the prediction issued for a given test example? (2) What features of these corpus examples are relevant for the model to relate them to the test example? SimplEx provides an answer by reconstructing the test latent representation as a mixture of corpus latent representations. Further, we propose a novel approach, the Integrated Jacobian, that allows SimplEx to make explicit the contribution of each corpus feature in the mixture. Through experiments on tasks ranging from mortality prediction to image classification, we demonstrate that these decompositions are robust and accurate. With illustrative use cases in medicine, we show that SimplEx empowers the user by highlighting relevant patterns in the corpus that explain model representations. Moreover, we demonstrate how the freedom in choosing the corpus allows the user to have personalized explanations in terms of examples that are meaningful for them.

\end{abstract}

\section{Introduction and related work}

How can we make a machine learning model convincing? If accuracy is undoubtedly necessary, it is rarely sufficient. As these models are used in critical areas such as medicine, finance and the criminal justice system, their black-box nature appears as a major issue~\cite{Lipton2016, Ching2018, Tjoa2020}. With the necessity to address this problem, the landscape of explainable artificial intelligence (XAI) developed~\cite{BarredoArrieta2020, Das2020}. A first approach in XAI is to focus on \emph{white-box models} that are interpretable by design. However, restricting to a class of inherently interpretable models often comes at the cost of lower prediction accuracy~\cite{Rai2019}. In this work, we rather focus on \emph{post-hoc explainability} techniques. These methods aim at improving the interpretability of black-box models by complementing their predictions with various kinds of explanations. In this way, it is possible to understand the prediction of a model without sacrificing its prediction accuracy.   

 \emph{Feature importance explanations} are undoubtedly the most widespread type of post-hoc explanations. Popular feature importance methods include SHAP~\cite{Shapley1953, Datta2016, Lundberg2017}, LIME~\cite{Ribeiro2016}, Integrated Gradients~\cite{Sundararajan2017}, Contrastive Examples \cite{Dhurandhar2018} and Masks~\cite{Fong2017, Fong2019,Crabbe2021}. These methods complement the model prediction for an input example with a score attributed to each input feature. This score reflects the importance of each feature for the model to issue its prediction.  Knowing which features are important for a model prediction certainly provides more information on the model than the prediction by itself. However, these methods do not provide a reason as to why the model pays attention to these particular features. 

Another approach is to contextualize each model prediction with the help of relevant examples. In fact, recent works~\cite{Nguyen2021} have demonstrated that human subjects often find example-based explanations more insightful than feature importance explanations. Complementing the model's predictions with relevant examples previously seen by the model is commonly known as \emph{Case-Based Reasoning} (CBR)~\cite{Caruana1999, Bichindaritz2006, Keane2019}. The implementations of CBR generally involve models that create a synthetic representation of the dataset, where examples with similar patterns are summarized by prototypes~\cite{Kim2015, Kim2016, Gurumoorthy2017}. At inference time, these models relate new examples to one or several prototypes to issue a prediction. In this way, the patterns that are used by the model to issue a prediction are made explicit with the help of relevant prototypes. A limitation of this approach is the restricted model architecture. The aforementioned procedure requires to opt for a family of models that rely on prototypes to issue a prediction. This family of model might not always be the most suitable for the task at hand. This motivates the development of generic post-hoc methods that make few or no assumption on the model. 

The most common approach to provide example-based explanations for a wide variety of models mirrors feature importance methods. The idea is to complement the model prediction by attributing a score to each training example. This score reflects the importance of each training example for the model to issue its prediction. This  score will typically be computed by simulating the effect of removing each training instance from the training set on the learned model~\cite{Cook1982}. Popular examples of such methods include Influence Functions~\cite{Koh2017} and Data-Shapley~\cite{Ghorbani2019, Ghorbani2020}. These methods offer the advantage of being flexible enough to be used with a wide variety of models. They produce scores that describe what the model could have predicted if some examples were absent from the training set. This is very interesting in a data valuation perspective. However, in an explanation perspective, it is not clear how to reconstruct the model predictions with these importance scores.  

So far, we have only discussed works that provide explanations of a model output, which is the tip of the iceberg. Modern machine learning models involve many convoluted transformations to deduce the output from an input. These transformations are expressed in terms of intermediate variables that are often called \emph{latent variables}. Some treatment of these latent variables is necessary if we want to provide explanations that take the model complexity into account. This motivates several works that push the explainability task beyond the realm of model outputs. Among the most noticeable contributions in this endeavour, we cite \emph{Concept Activation Vectors} that create a dictionary between human friendly concepts (such as the presence of stripes in an image) and their representation in terms of latent vectors~\cite{Kim2017}. Another interesting contribution is the \emph{Deep k-Nearest Neighbors} model that contextualizes the prediction for an example with its Nearest Neighbours in the space of latent variables, the \emph{latent space}~\cite{Papernot2018}. An alternative exploration of the latent space is offered by the \emph{representer theorem} that allows, under restrictive assumptions, to use latent vectors to decompose a model's prediction in terms of its training examples~\cite{Yeh2018}. 

\begin{wrapfigure}{r}{0.7\textwidth} 
\vspace{-0.9cm}	
\begin{center}
\includegraphics[width=0.7\textwidth]{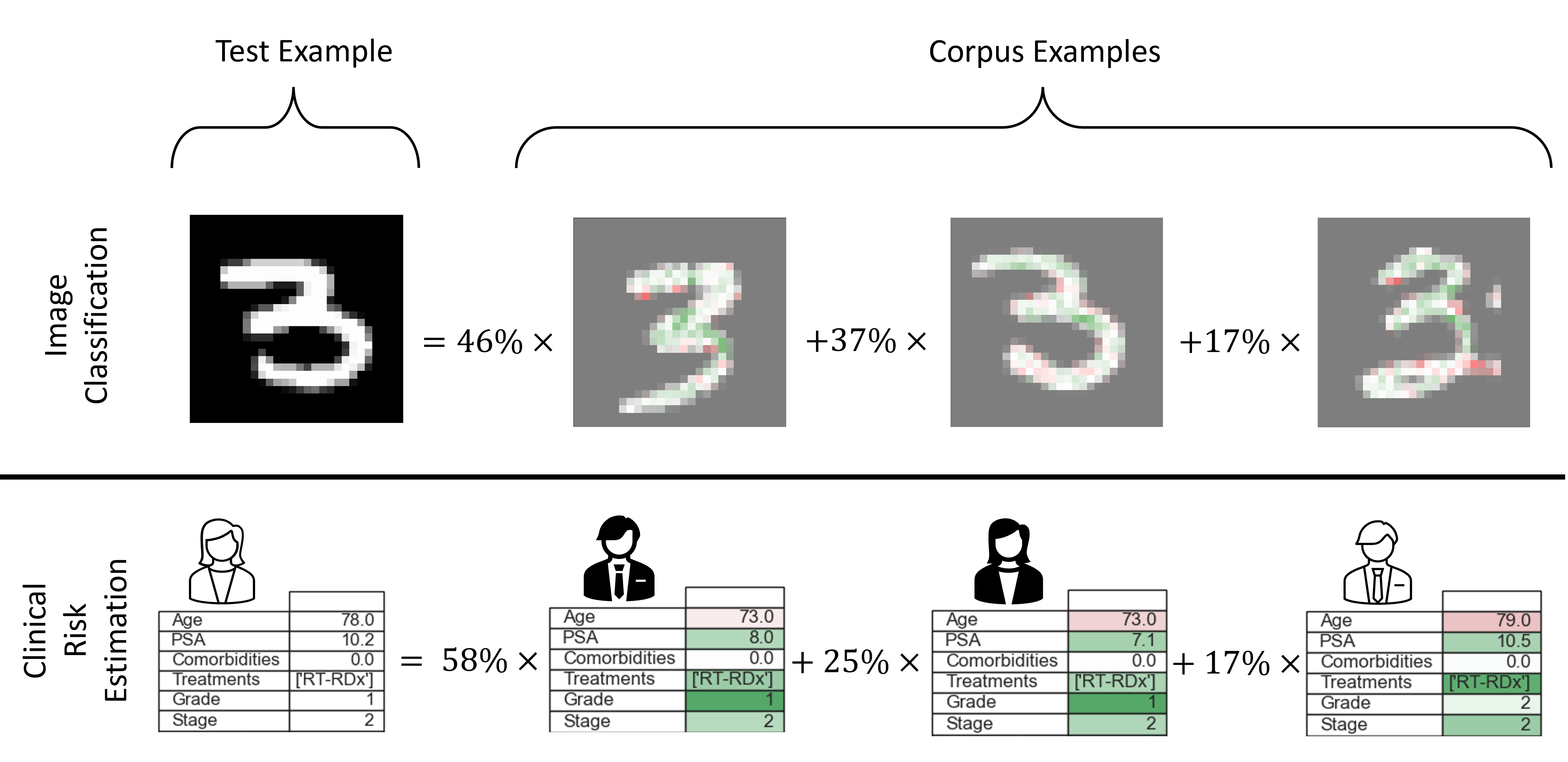}
\end{center}
\vspace{-0.4cm}
\caption{An example of corpus decomposition with SimplEx.}
\label{fig:simplex_examples}
\vspace{-0.5cm}
\end{wrapfigure}

\textbf{Contribution} In this work, we introduce a novel approach called SimplEx that lies at the crossroad of the above research directions.  SimplEx outputs post-hoc explanations in the form of Figure~\ref{fig:simplex_examples}, where the model's prediction and latent representation for a test example is approximated as a mixture of examples extracted from a \emph{corpus} of examples. In each case, SimplEx highlights the role played by each feature of each corpus example in the latent space decomposition. SimplEx centralizes many functionalities that, to the best of our knowledge, constitute a leap forward from the previous state of the art. (1)~SimplEx gives the user freedom to choose the corpus of examples whom with the model's predictions are decomposed. Unlike previous methods such as the representer theorem, there is no need for this corpus of examples to be equal to the model's training set. This is particularly interesting for two reasons: (a) the training set of a model is not always accessible (b) the user might want explanations in terms of examples that make sense for them. For instance, a doctor might want to understand the predictions of a risk model in terms of patients they know. (2)~The decompositions of SimplEx are valid, both in latent and output space. We show that, in both cases, the corpus mixtures discovered by SimplEx offer significantly more precision and robustness than previous methods such as Deep k-Nearest Neighbors and the representer theorem. (3)~SimplEx details the role played by each feature in the corpus mixture. This is done by introducing Integrated Jacobians, a generalization of Integrated Gradients that makes the contribution of each corpus feature explicit in the latent space decomposition. This creates a bridge between two research directions that have mostly developed independently: feature importance and example-based explanations~\cite{Keane2019, Keane2019b}.
In  Section~\ref{sec:user_study} of the supplementary material, we report a user-study involving 10 clinicians. This study supports the significance of our contribution.

\iffalse
--------------------------------------------------------
 For the time being, it is sufficient to interpret green and red features as the ones that respectively increase and decrease the quality of the corpus approximation for the test example in latent space.

similarity = reconstruct test prediction as a linear combination of examples prediction ; difference = not relying on optimality, no assumption made on training (regularization), mixtures are easier to understand than weights, we have a rigorous way to detect when the corpus does not explain the test prediction. We can make an experiment comparing both methods with and without regularization (the representer relies on the assumption that the model was trained with a $L^2$ regularization).

Synthetic control \cite{Abadie2010, Athey2017, Amjad2018, Abadie2020}.

\begin{figure} 
\begin{center}
\includegraphics[width=0.8\textwidth]{Images/simplex_examples}
\end{center}
\vspace{-0.2cm}
\caption{An example of corpus decomposition with SimplEx.}
\label{fig:simplex_examples}
\vspace{-0.5cm}
\end{figure}

\fi

\section{SimplEx} \label{sec:problem}

In this section, we formulate our method rigorously. Our purpose is to explain the black-box prediction for an unseen test example with the help of a set of known examples that we call the \emph{corpus}. We start with a clear statement of the family of black-boxes for which our method applies. Then, we detail how the set of corpus examples can be used to decompose a black-box representation for the unseen example. Finally, we show that the corpus decomposition can offer explanations at the feature level. 

\subsection{Preliminaries}
Let $\X \subseteq \mathbb{R}^{d_X}$ be an input (or feature) space and $\Y \subseteq \mathbb{R}^{d_Y}$ be an output (or label) space, where $d_X$ and $d_Y$ are respectively the dimension of the input and the output space. Our task is to explain individual predictions of a given black-box $\f : \X \rightarrow \Y$. In order to build our explainability method, we need to make an assumption on the family of black-boxes that we wish to interpret. 
\begin{assumption}[Black-box Restriction] \label{assumption-linear}
We restrict to black-boxes $\f : \X \rightarrow \Y$ that can be decomposed as $\f = \l \circ \g $, where $\g : \X \rightarrow \H $ maps an input $\x \in \X$ to a latent vector $\h = \g\left(\x \right) \in \H$ and $\l : \H \rightarrow \Y $ linearly maps\footnote{The map can in fact be affine. In the following, we omit the bias term $\textbf{b} \in \Y$ that can be reabsorbed in $\g$.} a latent vector $\h \in \H$ to an output $\y = \l(\h) = \A \h \in \Y$. In the following, we call $\H \subseteq \R^{d_H}$ the \emph{latent space}. Typically, this space has higher dimension than the output space $d_H > d_Y$.
\end{assumption}
\begin{remark}
In the context of deep-learning, this assumption requires that the last hidden layer maps linearly to the output. While it is often the case, it is crucial in the following since we will use the fact that linear combinations in latent space correspond to linear combinations in output space. Our purpose is to gain insights on the structure of the latent space.
\end{remark}
\begin{remark}
This assumption is compatible with regression and classification models, we just need to clarify what we mean by \emph{output} in the case of classification. If $\f$ is a classification black-box that predicts the probabilities for each class, it will typically take the form in Assumption~\ref{assumption-linear} up to a normalizing map $\boldsymbol{\phi}$ (typically a softmax): $\f = \boldsymbol{\phi} \circ \l \circ \g $. In this case, we ignore\footnote{There is no loss of information as the output allows us to reconstruct class probabilities $\textbf{p}$ via $\textbf{p} = \boldsymbol{\phi} \left( \y \right) $.} the normalizing map $\boldsymbol{\phi}$ and define the output to be $\y = (\l \circ \g) \left( \x \right)$.  
\end{remark}
Our explanations for $\f$ rely on a set of examples that we call the corpus. These examples will typically (but not necessarily) be a representative subset of the black-box training set. The corpus set has to be understood as a set of reference examples that we want to use as building blocks to interpret unseen examples. In order to index these examples, it will be useful to denote by $[n_1:n_2]$ the set of natural numbers between the natural numbers $n_1$ and $n_2$ with $ n_1 < n_2$. Further, we denote $[n] = [1:n]$ the set of natural numbers between $1$ and $n \geq 1$. The corpus of examples is a set $\C = \left\{ \x^c \mid c \in [C]  \right\}$ containing $C \in \N^*$ examples $\x^c \in \X$. In the following, superscripts are labels for examples and subscripts are labels for vector components. In this way, $x_i^c$ has to be understood as the component $i$ of corpus example $c$. 

\subsection{A corpus of examples to explain a latent representation}
Our purpose is to understand a prediction $\f(\x)$ for an unseen test example $\x$ with the help of the corpus. How can we decompose the prediction $\f(\x)$ in terms of corpus predictions $\f(\x^c)$? A naive attempt would be to express $\x$ as a mixture of inputs from the corpus $\C$: $\x = \sum_{c=1}^C w^c \x^c$ with weights $w^c \in [0,1]$ that sum to one $\sum_{c=1}^C w^c = 1$. The weakness of this approach is that the signification of the mixture weights is not conserved if the black-box $\f$ is not a linear map: $ \f( \sum_{c=1}^Cw^c \x^c ) \neq \sum_{c=1}^C w^c \f(\x^c)$.

Fortunately, Assumption~\ref{assumption-linear} offers us a better vector space to perform a corpus decomposition of the unseen example $\x$. We first note that the map $\g$ induces a latent representation of the corpus $\g(\C)~=~\left\{\h^c =  \g(\x^c) \mid \x^c \in \C \right\} \subset \H$. Similarly, $\x$ has a latent representation $\h = \g(\x) \in \H$. Following the above line of reasoning, we could therefore perform a corpus decomposition in latent space $\h = \sum_{c=1}^C w^c \h^c$. Now, by using the linearity of $\l$, we can compute the black-box output of this mixture in latent space: $ \l( \sum_{c=1}^Cw^c \h^c ) = \sum_{c=1}^C w^c \l(\h^c)$. In this case, the weights that are used to decompose the latent representation $\h$ in terms of the latent representation of the corpus $\g(\C)$ also reflect the way in which the black-box prediction $\f(\x)$ can be decomposed in terms of the corpus outputs $\f (\C)$. This hints that the latent space $\H$ is endowed with the appropriate geometry to make corpus decompositions. More formally, we think in terms of the convex hull spanned by the corpus. 
\begin{definition}[Corpus Hull]
The \emph{corpus convex hull} spanned by a corpus $\C$ with latent representation $\g(\C)~=~\left\{\h^c =  \g(\x^c) \mid \x^c \in \C \right\} \subset \H$ is the convex set 
\begin{align*}
\CH \left( \C \right) = \left\{ \sum_{c=1}^C w^c \h^c \ \middle| \ w^c \in [0,1] \ \forall c \in [C] \ \wedge \ \sum_{c=1}^C w^c = 1 \right\}.
\end{align*}
\end{definition}
\begin{remark}
This is the set of latent vectors that are a mixture of the corpus latent vectors.
\end{remark}
At this stage, it is important to notice that an exact corpus decomposition is not possible if $\h \not\in \CH (\C)$. In such a case, the best we can do is to find the element $\hat{\h} \in \CH (\C)$ that best approximates $\h$. If $\H$ is endowed with a norm $\parallel \cdot \parallel_{\H}$, this corresponds to the convex optimization problem
\begin{align}\label{equ-CH_problem}
\hat{\h} = \argmin_{\tilde{\h} \ \in \ \CH (\C)} \parallel \h - \tilde{\h} \parallel_{\H}.
\end{align}
By definition, the corpus representation $\hat{\h}$ of $\h$ can be expanded\footnote{Note that this decomposition might not be unique, more details in Section~\ref{sec:problem_sup} of the supplementary material.} as a mixture of elements from $\g (\C)$: $\hat{\h} = \sum_{c=1}^C w^c \h^c$. The weight can naturally be interpreted as a measure of importance in the reconstruction of $\h$ with the corpus. Clearly, $w^c \approx 0$ for some $c \in [C]$ indicates that $\h^c$ does not play a significant role in the corpus representation $\hat{\h}$ of $\h$. On the other hand,  $w^c \approx 1$ indicates that $\h^c$ generates the corpus representation $\hat{\h}$ by itself. 

At this stage, a natural question arises: how can we know if the corpus approximation $\hat{\h}$ is a good approximation for $\h$? The answer is given by the residual vector $\h - \hat{\h}$ that measures the shift between the latent representation $\h = \g (\x)$ and the corpus hull $\CH (\C)$. It is natural to use this residual vector to detect examples that cannot be explained with the selected corpus of examples $\C$.
\begin{definition}[Corpus Residual]\label{def-corpus_residual}
The \emph{corpus residual} associated to a latent vector $\h \in \H$ and its corpus representation $\hat{\h} \in \CH(\C)$ solving \eqref{equ-CH_problem} is the quantity
\begin{align*}
r_{\C}(\h) = \parallel \h - \hat{\h} \parallel_{\H}  = \min_{\tilde{\h} \ \in \ \CH (\C)}\parallel \h - \tilde{\h} \parallel_{\H}.
\end{align*}
\end{definition}

\begin{wrapfigure}{r}{0.5\textwidth} 
\vspace{-.9cm}
\begin{center} 
\includegraphics[width=0.48\textwidth]{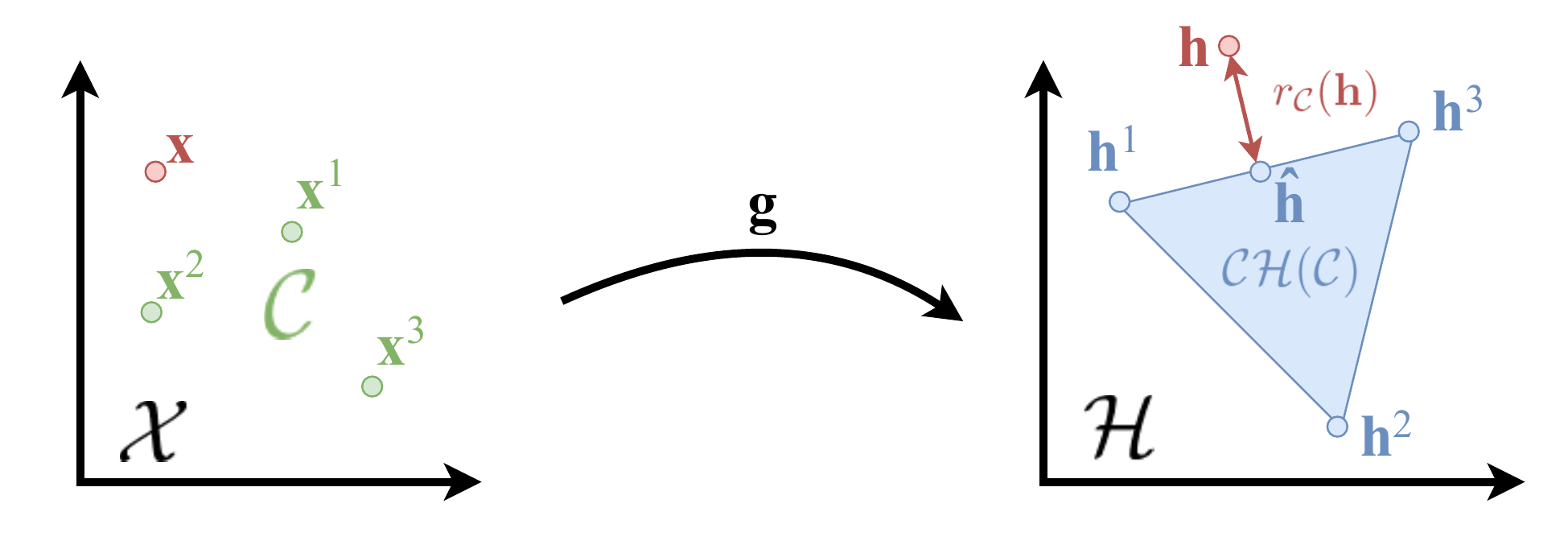}
\end{center}
\vspace{-.3cm}
\caption{Corpus convex hull and residual.}
\vspace{-.4cm}
\label{fig-corpus}
\end{wrapfigure}
In Section~\ref{subsec:precision_output} of the supplementary material, we show that the corpus residual also controls the quality of the corpus approximation in output space $\Y$. All the corpus-related quantities that we have introduced so far are summarized visually in Figure~\ref{fig-corpus}. Note that this Figure is a simplification of the reality as $C$ will typically be larger than 3 and $d_X, d_H$ will typically be higher than 2. We are now endowed with a rigorous way to decompose a test example in terms of corpus examples \emph{in latent space}. In the next section, we detail how to pull-back this decomposition to input space.

\subsection{Transferring the corpus explanation in input space}

Now that we are endowed with a corpus decomposition $\hat{\h} = \sum_{c=1}^C w^c \h^c$ that approximates $\h$, it would be convenient to have an understanding of the corpus decomposition in input space $\X$. For the sake of notation, we will assume that the corpus approximation is good so that it is unnecessary to draw a distinction between the latent representation $\h$ of the unseen example $\x$ and its corpus decomposition $\hat{\h}$. If we want to understand the corpus decomposition in input space, a natural approach~\cite{Sundararajan2017} is to fix a baseline input $\x^0$ together with its latent representation $\h^0 = \g(\x^0)$. Let us now decompose the representation shift $\h - \h^0$ in terms of the corpus:
\begin{align} \label{equ-shift_corpus_decomposition}
\h - \h^0 =  \sum_{c=1}^C w^c \left( \h^c - \h^0 \right). 
\end{align}
With this decomposition, we understand the total shift in latent space $\h - \h^0$ in terms of individual contributions from each corpus member. In the following, we focus on the comparison between the baseline and a single corpus example $\x^c$ together with its latent representation $\h^c$ by keeping in mind that the full decomposition \eqref{equ-shift_corpus_decomposition} can be reconstructed with the whole corpus. To bring the discussion in input space $\X$, we interpret the shift in latent space $\h^c - \h^0$ as resulting from a shift $\x^c - \x^0$ in the input space. We are interested in the contribution of each feature to the latent space shift. To decompose the shift in latent space in terms of the features, we parametrize the shift in input space with a line $\linec: [0,1] \rightarrow \X$ that goes from the baseline to the corpus example: $\linec (t) = \x^0 + t \cdot (\x^c - \x^0)$ for $t \in [0,1]$. Together with the black-box, this line induces a curve in latent space $\curvec: [0,1] \rightarrow \H$ that goes from the baseline latent representation $\h^0$ to the corpus example latent representation $\h^c$. Let us now use an infinitesimal decomposition of this curve to make the contribution of each input feature explicit. If we assume that $\g$ is differentiable at $\linec(t)$, we can use a first order approximation of the curve at the vicinity of $t \in (0,1)$ to decompose the infinitesimal shift in latent space:
\begin{align*}
\underbrace{\curvec (t + \delta t) - \curvec (t)}_{\text{Infinitesimal shift in latent space}} & =  \sum_{i= 1 }^{d_X} \frac{\partial \g}{\partial x_i} \bigg|_{\linec(t)} \frac{d \gamma^c_i}{dt}\bigg|_{t} \ \delta t + o(\delta t) \\
& =  \sum_{i= 1 }^{d_X} \frac{\partial \g}{\partial x_i} \bigg|_{\linec(t)} (x^c_i - x^0_i) \cdot \delta t + o(\delta t),
\end{align*}    
where we used $\gamma^c_i(t) = x^0_i + t \cdot (x^c_i - x^0_i)$ to obtain the second equality. In this decomposition, each input feature contributes additively to the infinitesimal shift in latent space. It follows trivially that the contribution of the input feature corresponding to input dimension $i \in [d_X]$ is given by
\begin{align*}
\delta \textbf{j}^c_i (t) =  (x^c_i - x^0_i) \cdot \frac{\partial \g}{\partial x_i} \bigg|_{\linec(t)} \delta t \hspace{0.5cm} \in \H.
\end{align*}
In order to compute the overall contribution of feature $i$ to the shift, we let $\delta t \rightarrow 0$ and we sum the infinitesimal contributions along the line $\linec$. If we assume\footnote{This is not restrictive, DNNs with ReLU activation functions satisfy this assumption for instance.} that $\g$ is almost everywhere differentiable, this sum converges to an integral in the limit $\delta t \rightarrow 0$ . This motivates the following definitions.

\begin{definition}[Integrated Jacobian \& Projection] The \emph{integrated Jacobian} between a baseline $(\x^0, \h^0 = \g(\x^0))$ and a corpus example $(\x^c, \h^c = \g(\x^c)) \in \X \times \H$ associated to feature $i \in [d_X]$ is
\begin{align*}
\textbf{j}_i^c = \left( x^c_i - x^0_i \right) \int_0^1 \frac{\partial \g}{\partial x_i} \bigg|_{\linec(t)} \ dt \hspace{0.5cm} \in \H,
\end{align*}
where $\linec (t) \equiv \x^0 + t \cdot \left( \x^c - \x^0 \right)$ for $t \in [0,1]$. This vector indicates the shift in latent space induced by feature $i$ of corpus example $c$ when comparing the corpus example with the baseline. To summarize this contribution to the shift $\h - \h^0$ described in \eqref{equ-shift_corpus_decomposition}, we define the \emph{projected Jacobian}
\begin{align*}
p^c_i = \proj_{\h - \h^0} \left( \textbf{j}_i^c \right) \equiv \frac{\langle \ \h - \h^0 \ , \ \textbf{j}_i^c \ \rangle}{\langle \ \h - \h^0 \ , \ \h - \h^0 \ \rangle} \hspace{0.5cm} \in \R,
\end{align*}  
where  $\langle\cdot{,}\cdot\rangle$ is an inner product for  $\H$ and the normalization is chosen for the purpose of Proposition~\ref{prop:ij_properties}.
\end{definition}

\begin{remark}
The integrated Jacobian can be seen as a latent-space generalization of Integrated Gradients~\cite{Sundararajan2017}. In Section~\ref{subsec:integrated_gradient} of the supplementary material, we establish the relationship between the two quantities: $\text{IG}_i^c = \text{l}(\textbf{j}_i^c)$.
\end{remark}

\begin{wrapfigure}{r}{0.5\textwidth} 
\begin{center} 
\includegraphics[width=0.5\textwidth]{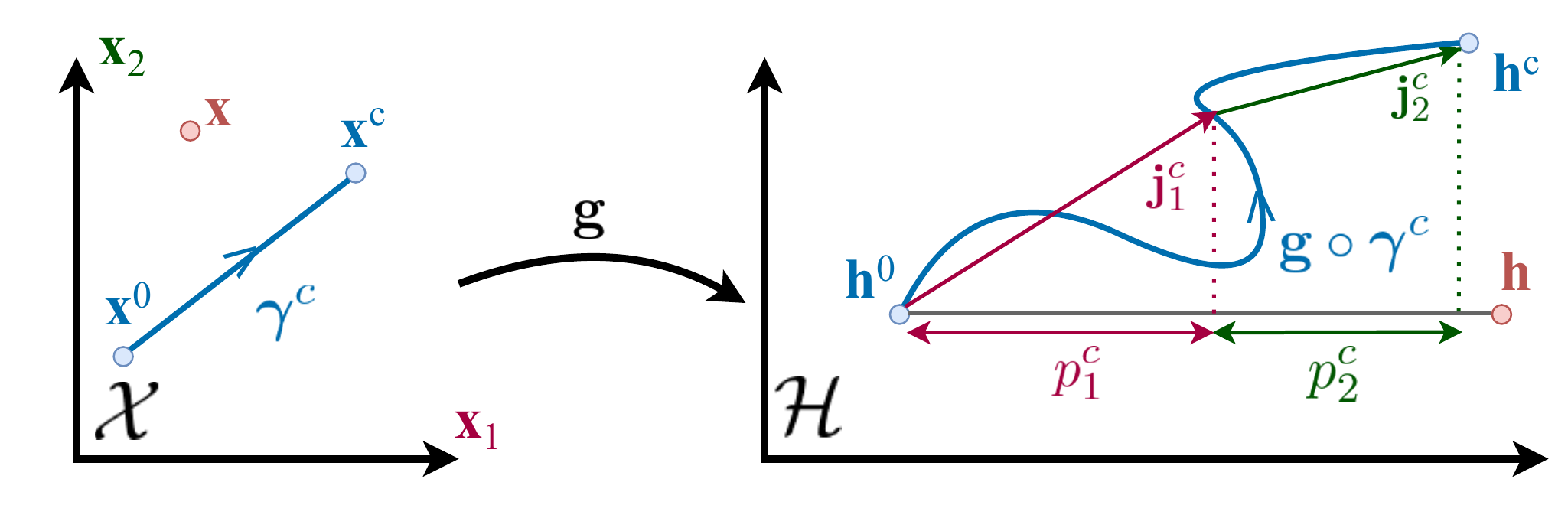}
\end{center}
\caption{Integrated Jacobian and projection.}
\label{fig:jacobian}
\end{wrapfigure}
We summarize the Jacobian quantities in Figure~\ref{fig:jacobian}. By inspecting the figure, we notice that projected Jacobians encode the contribution of feature $i$ from corpus example $c$ to the overall shift in latent space: $p^c_i > 0$ implies that this feature creates a shift pointing in the same direction as the overall shift; $p^c_i < 0$ implies that this feature creates a shift pointing in the opposite direction and $p^c_i = 0$ implies that this feature creates a shift in an orthogonal direction. We use the projections to summarize the contribution of each feature in Figures~\ref{fig:simplex_examples} , \ref{fig:misclassified1} \& \ref{fig:misclassified2}. The colors blue and red indicate respectively a positive and negative projection. In addition to these geometrical insights, Jacobian quantities come with natural properties.

\begin{proposition}[Properties of Integrated Jacobians] \label{prop:ij_properties}
Consider a baseline $(\x^0, \h^0 = \g(\x^0))$ and a test example together with their latent representation $(\x, \h = \g(\x)) \in \X \times \H$. If the shift $\h - \h^0$ admits a decomposition~\eqref{equ-shift_corpus_decomposition}, the following properties hold.\\
\begin{align*}
(A):~\sum_{c=1}^C \sum_{i=1}^{d_X} w^c \textbf{j}^c_i = \h - \h^0 \hspace{1cm} (B):~\sum_{c=1}^C \sum_{i=1}^{d_X} w^c p^c_i = 1.
\end{align*} 
\end{proposition}
\begin{proof}
The proof is provided in Section~\ref{subsec:proof_prop} of the supplementary material. 
\end{proof}

These properties show that the integrated Jacobians and their projections are the quantities that we are looking for: they transfer the corpus explanation into input space. The first equality decomposes the shift in latent space in terms of contributions $w^c \textbf{j}^c_i$ arising from each feature of each corpus example. The second equality sets a natural scale to the contribution of each feature. For this reason, it is natural to use $w^c p^c_i$ to measure the contribution of feature $i$ of corpus example $c$.

\section{Experiments} \label{sec:experiment}

In this section, we evaluate quantitatively several aspects of our method. In a first experiment, we verify that the corpus decomposition scheme described in Section~\ref{sec:problem} yields good approximations for the latent representation of test examples extracted from the same dataset as the corpus examples. In a realistic clinical use case, we illustrate the usage of SimplEx in a set-up where different corpora reflecting different datasets are used. The experiments are summarized below. In Section~\ref{sec:experiments_sup} of the supplementary material, we provide more details and further experiments with time series and synthetic data. The code for our method and experiments is available on the Github repository \url{https://github.com/JonathanCrabbe/Simplex}. All the experiments have been replicated on different machines. 

\subsection{Precision of corpus decomposition} \label{subsec:precision_experiment}

\begin{figure}
\begin{center}
  \begin{subfigure}{.4\textwidth}
  \centering
  \includegraphics[width=\linewidth]{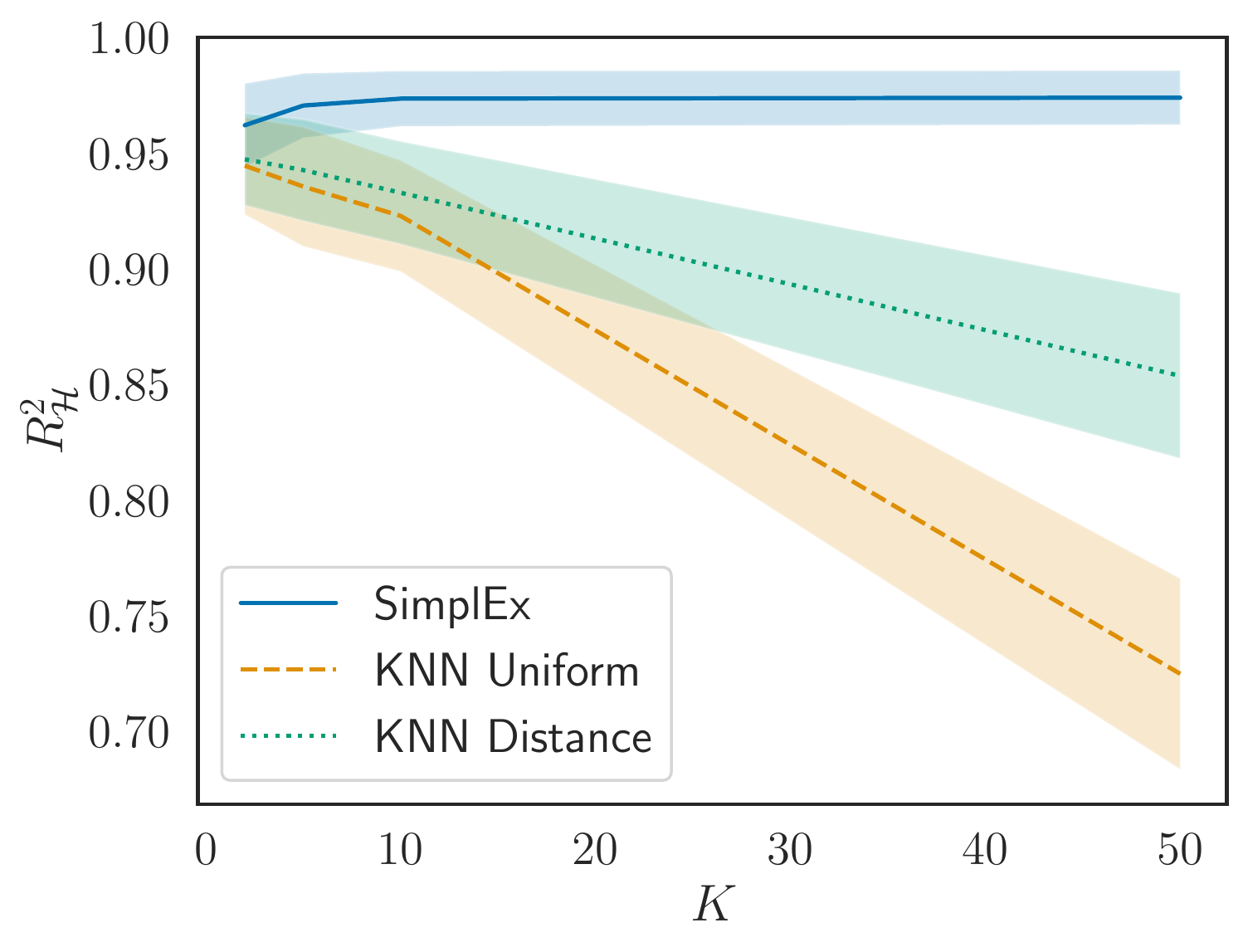}
  \caption{$R^2_{\H}$ score for the latent approximation}
  \label{fig:prostate_quality_a}
\end{subfigure}
\begin{subfigure}{.4\textwidth}
  \centering
  \includegraphics[width=\linewidth]{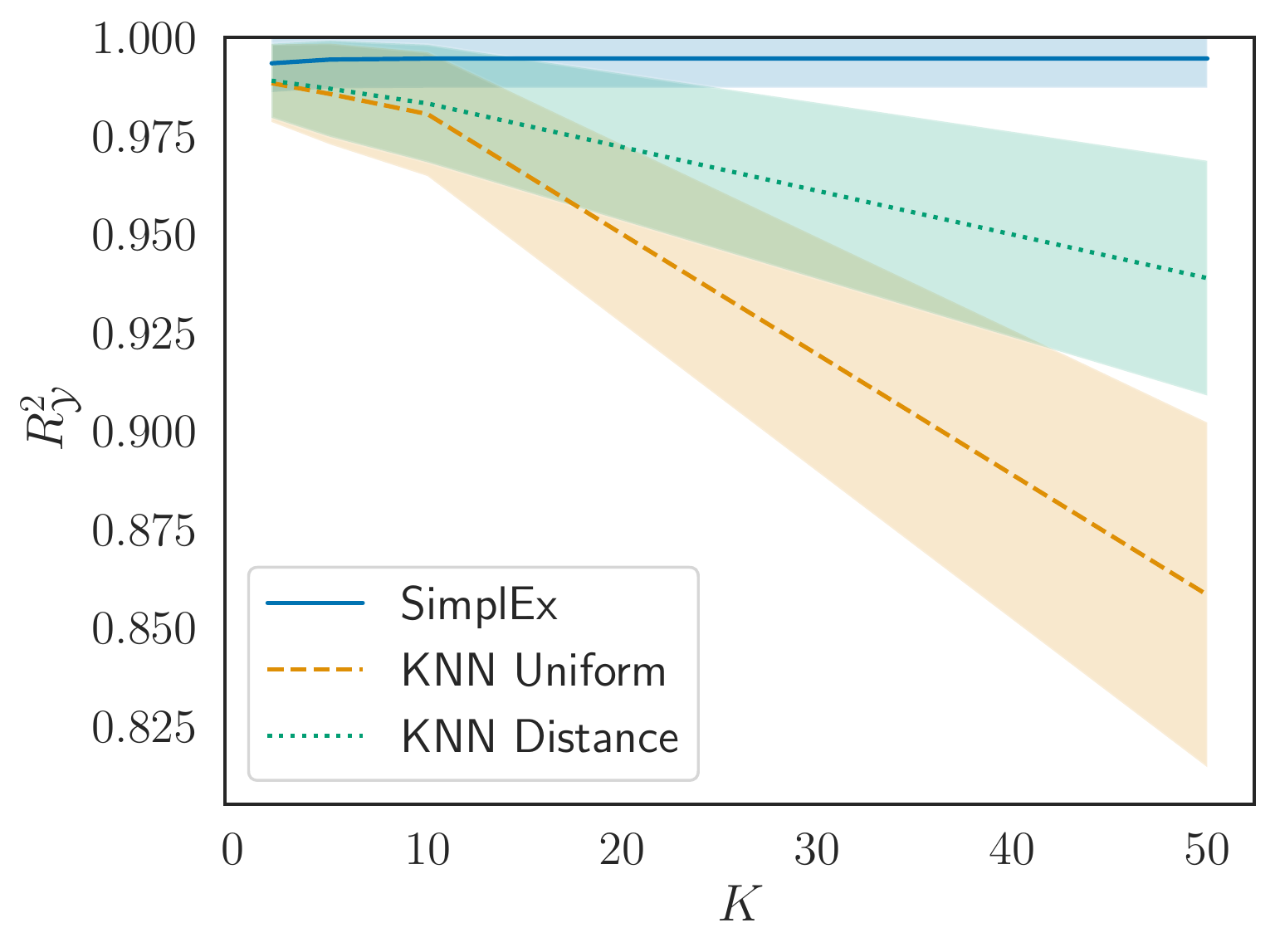}
  \caption{$R^2_{\Y}$ score for the output approximation}
  \label{fig:prostate_quality_b}
\end{subfigure}
\end{center}
\caption{Precision of corpus decomposition for prostate cancer (avg $\pm$ std).}
\label{fig:prostate_quality}
\end{figure}

\begin{figure}
\begin{center}
  \begin{subfigure}{.4\textwidth}
  \centering
  \includegraphics[width=\linewidth]{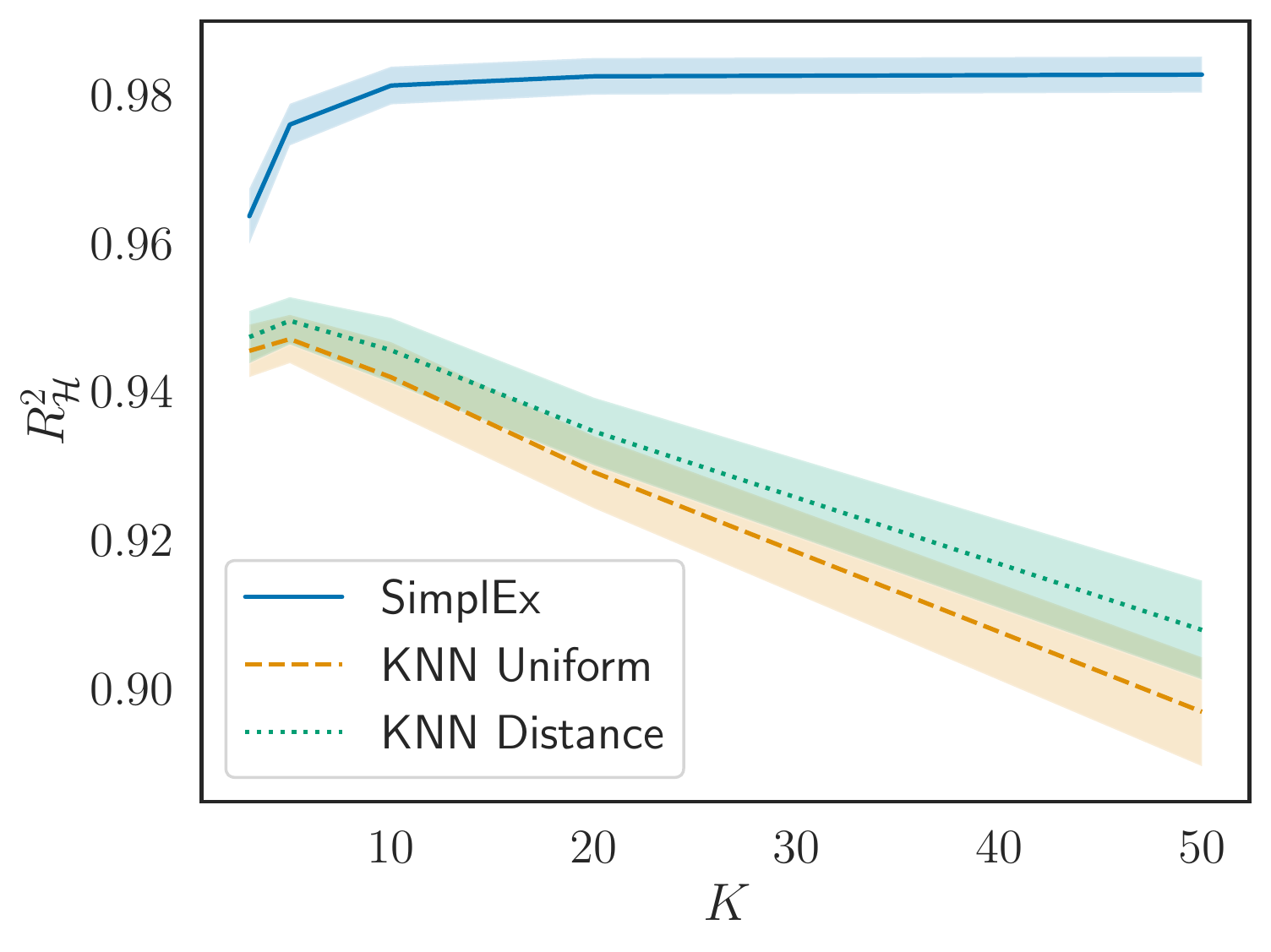}
  \caption{$R^2_{\H}$ score for the latent approximation}
  \label{fig:mnist_quality_a}
\end{subfigure}%
\begin{subfigure}{.4\textwidth}
  \centering
  \includegraphics[width=\linewidth]{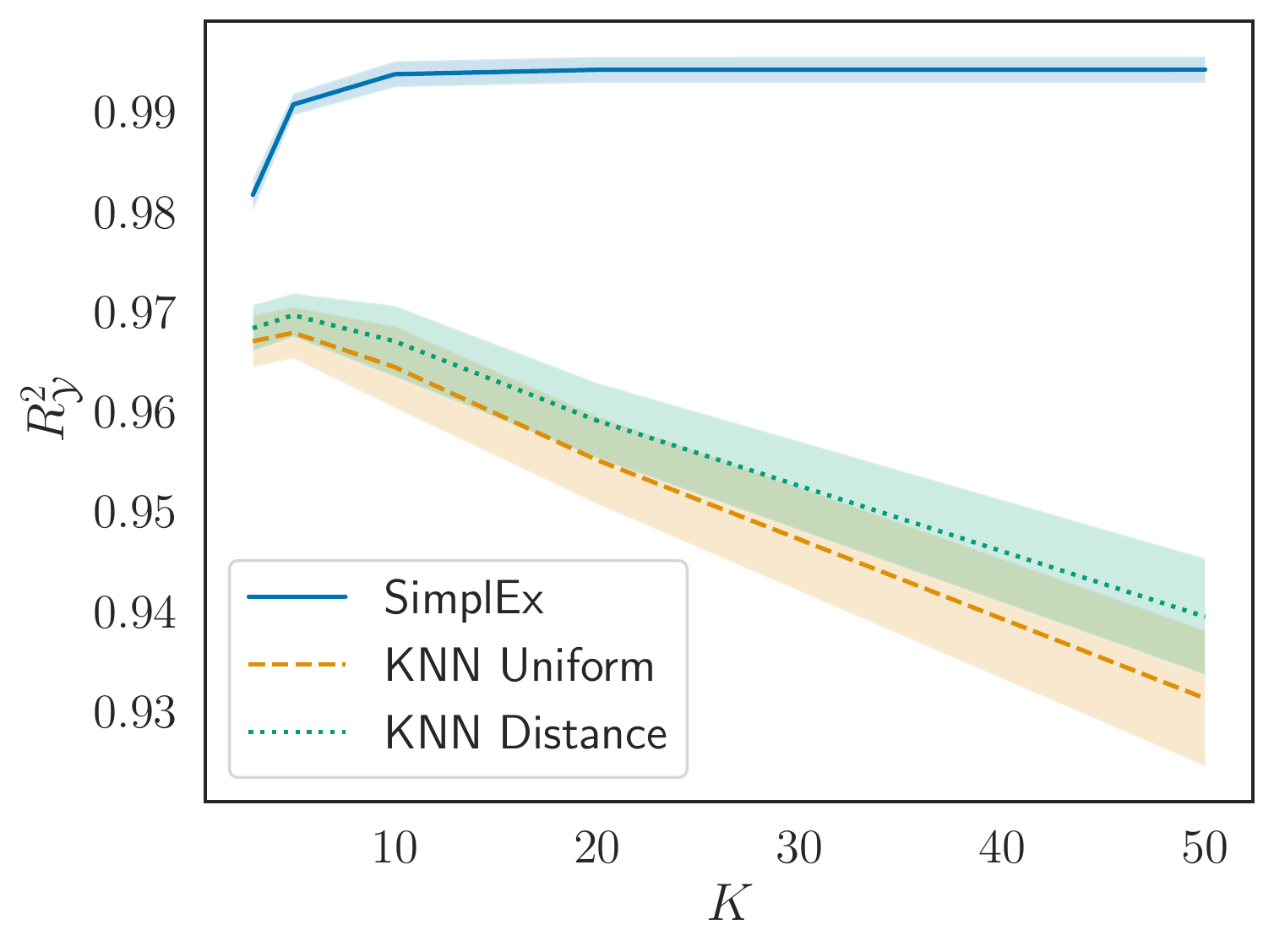}
  \caption{$R^2_{\Y}$ score for the output approximation}
  \label{fig:mnist_quality_b}
\end{subfigure}
\end{center}
\caption{Precision of corpus decomposition for MNIST (avg $\pm$ std).}
\label{fig:mnist_quality}
\vspace{-0.5cm}
\end{figure}

\textbf{Description} The purpose of this experiment is to check if the corpus decompositions described in Section~\ref{sec:problem} allows us to build good approximations of the latent representation of test examples. We start with a dataset $\D$ that we split into a training set $\Dtrain$ and a testing set $\Dtest$. We train a black-box $\f$ for a given task on the training set $\Dtrain$. We randomly sample  a set of corpus examples from the training set $\C \subset \Dtrain$ (we omit the true labels for the corpus examples) and a set of test examples from the testing set $\T \subset \Dtest$. For each test example $\x \in \T$, we build an approximation $\hat{\h}$ for $\h = \g (\x)$ with the corpus examples latent representations. In each case, we let the method use only $K$ corpus examples to build the approximation. We repeat the experiment for several values of $K$.   

\textbf{Metrics} We are interested in measuring the precision of the corpus approximation in latent space and in output space. To that aim, we use the $R^2$ score in both spaces. In this way, $R^2_{\H}$ measures the precision of the corpus approximation $\hat{\h}$ with respect to the true latent representation $\h$. Similarly, $R^2_{\Y}$ measures the precision of the corpus approximation $\hat{\y} =\l (\hat{\h})$ with respect to the true output $\y = \l(\h)$. Both of these metrics satisfy $-\infty < R^2 \leq 1$. A higher $R^2$ score is better with $R^2 = 1$ corresponding to a perfect approximation. All the metrics are computed over the test examples $\T$. The experiments are repeated 10 times to report standard deviations across different runs. 

\textbf{Baselines} We compare our method\footnote{To enforce SimplEx to select $K$ examples, we add a $L^1$ penalty making the $C - K$ smallest weights vanish.} (SimplEx) with 3 baselines.  A first approach, inspired by~\cite{Papernot2018}, consists in using the $K$-nearest corpus neighbours \emph{in latent space} to build the latent approximation $\hat{\h}$. Building on this idea, we introduce two baselines (1)~KNN Uniform that takes the average latent representation of the $K$-nearest corpus neighbours of $\h$ in latent space (2)~KNN Distance that computes the same average with weights $w^c$ inversely proportional to the distance $\parallel \h - \h^c \parallel_{\H}$. Finally, we use the representer theorem~\cite{Yeh2018} to produce an approximation $\hat{\y}$ of $\y$ with the corpus $\C$. Unlike the other methods, the representer theorem does not allow to produce an approximation in latent space. 

\textbf{Datasets} We use two different datasets with distinct tasks for our experiment: (1)~240,486 patients enrolled in the American SEER program~\cite{Seer2019}. We consider the binary classification task of predicting cancer mortality for patients with prostate cancer. We train a multilayer perceptron (MLP) for this task. Since this task is simple, we show that a corpus of $C=100$ patients yields good approximations. (2)~70,000 MNIST images of handwritten digits~\cite{Deng2012}. We consider the multiclass classification task of identifying the digit represented on each image. We train a convolutional neural network (CNN) for the image classification. This classification task is more complex than the previous one (higher $d_X$ and $d_Y$), we show that a corpus of $C=1,000$ images yields good approximations in this case.

\textbf{Results} The results for SimplEx and the KNN baselines are presented in Figure~\ref{fig:prostate_quality} \& \ref{fig:mnist_quality}. Several things can be deduced from these results: (1) It is generally harder to produce a good approximation in latent space than in output space as $R^2_{\H} < R^2_{\Y}$ for most examples (2) SimplEx produces the most accurate approximations, both in latent and output space. These approximations are of high quality with $R^2 \approx 1$. (3) The trends are qualitatively different between SimplEx and the other baselines. The accuracy of SimplEx increases with $K$ and stabilizes when a small number of corpus members contribute ($K=5$ in both cases). The accuracy of the KNN baselines increases with $K$, reaches a maximum for a small $K$ and steadily decreases for larger $K$. This can be understood easily: when $K$ increases beyond the number of relevant corpus examples, irrelevant examples will be added in the decomposition. SimplEx will typically annihilate the effect of these irrelevant examples by setting their weights $w^c$ to zero in the corpus decomposition. The KNN baselines include the irrelevant corpus members in the decomposition, which alters the quality of the approximation. This suggests that $K$ has to be tuned for each example with KNN baselines, while the optimal number of corpus examples to contribute is learned by SimplEx. (4) The standard deviations indicate that the performances of SimplEx are more consistent across different runs. This is particularly true in the prostate cancer experiment, where the corpus size $C$ is smaller. This suggests that SimplEx is more robust than the baselines. (5) For the representer theorem, we have $R^2_{\Y} = -(6.6 \pm 6.1)\cdot 10^7$ for the prostate cancer dataset and $R^2_{\Y} = -(7.2 \pm 6.6)$ for MNIST. This corresponds to poor estimations of the black-box output. We propose some hypotheses to explain this observation in Section~\ref{subsec:precision_sup} of the supplementary material.  

\subsection{Significance of Jacobian Projections} 

\textbf{Description} The purpose of this experiment is to check if SimplEx's Jacobian Projections are a good measure of the importance for each corpus feature in constructing the test latent representation $\h$.
In the same setting as in the previous experiment, we start with a corpus $\mathcal{C}$ of $C = 500$ MNIST images. We build a corpus approximation for an example $\x \in \X$ with latent representation $\h = \g(\x) \in \H$. The precision of this approximation is reflected by its corpus residual $r_{\C}(\h)$. For each corpus example $\x^c \in \C$, we would like to identify the features that are the most important in constructing the corpus decomposition of $\h$. With SimplEx, this is reflected by the Jacobian Projections $p^c_i$. We evaluate these scores for each feature $i \in [d_X]$ of each corpus example  $c \in [C]$. For each corpus image $\textbf{x}^c \in \mathcal{C}$, we select the $n$ most important pixels according to the Jacobian Projections and the baseline. In each case, we build a mask $\textbf{m}^c$ that replaces these $n$ most important pixels by black pixels. This yields a corrupted corpus image $\textbf{x}^c_{cor} = \textbf{m}^c \odot \textbf{x}^c$, where $\odot$ denotes the Hadamard product. By corrupting all the corpus images, we obtain a corrupted corpus $\C_{cor}$ . We analyse how well this corrupted corpus approximates $\h$, this yields a residual $r_{\C_{cor}}(\h)$. 

\textbf{Metric} We are interested in measuring the effectiveness of the corpus corruption. This is reflected by the metric $\delta_{cor}(\h) = r_{\C_{cor}}(\h) - r_{\C}(\h)$ . A higher value for this metric indicates that the features selected by the saliency method are more important for the corpus to produce a good approximation of $\textbf{h}$ in latent space. We repeat this experiment for 100 test examples and for different numbers  $n$ of perturbed pixels. 

\textbf{Baseline} As a baseline for our experiment, we use Integrated Gradients, which is close in spirit to our method. In a similar fashion, we compute the Integrated Gradients $IG^c_i$ for each feature $i \in [d_X]$ of each corpus example  $c \in [C]$ and construct a corrupted corpus based on these scores.

\textbf{Results} The results are presented in the form of box plots in Figure~\ref{fig:jacobian_corruption}. We observe that the corruptions induced by the Jacobian Projections are significantly more impactful when few pixels are perturbed. The two methods become equivalent when more pixels are perturbed. This demonstrates that Jacobian Projections are more suitable to measure the importance of features when performing a latent space reconstruction, as it is the case for SimplEx.

\begin{figure}[h]
	\centering
	\includegraphics[width=.5\textwidth]{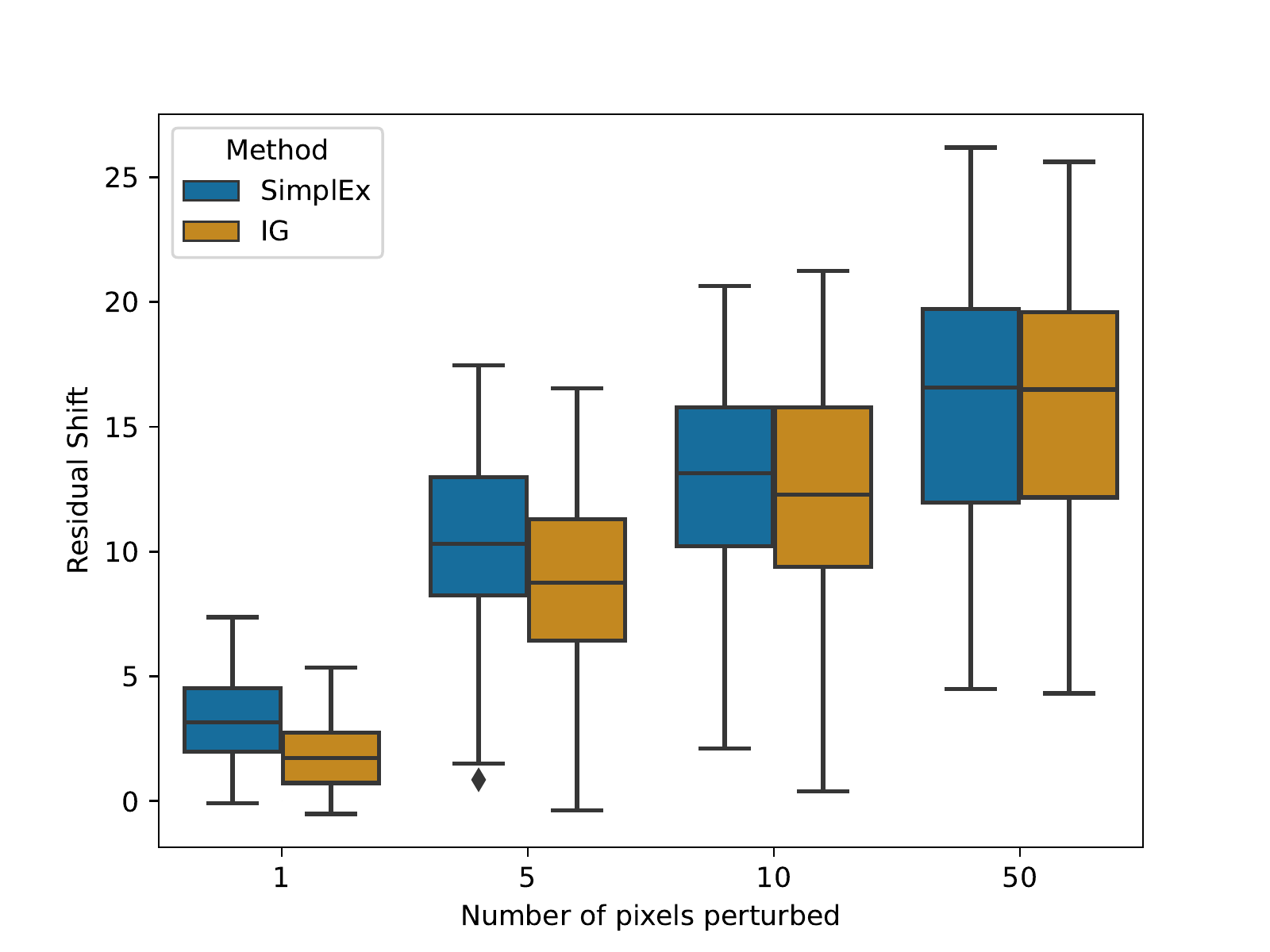}
	\caption{Increase in the corpus residual caused by each method (higher is better).}
	\label{fig:jacobian_corruption}
\end{figure}

\subsection{Use case: clinical risk model across countries} \label{subsec:clinical_use_case}

Very often, clinical risk models are produced and validated with the data of patients treated at a single site~\cite{Wu2021}. This can cause problems when these models are deployed at different sites for two reasons: (1)~Patients from different sites can have different characteristics (2)~Rules that are learned for one site might not be true for another site. One possible way to alleviate this problem would be to detect patients for which the model prediction is highly extrapolated and/or ambiguous. In this way, doctors from different sites can make an enlightened use of the risk model rather than blindly believing the model's predictions. We demonstrate that SimplEx provides a natural framework for this set-up.

As in the previous experiment, we consider a dataset $\Dusa$ containing patients enrolled in the American SEER program~\cite{Seer2019}. We train and validate an MLP risk model with $\Dusa$. To give a realistic realization of the above use-case, we assume that we want to deploy this risk model in a different site: the United Kingdom. For this purpose, we extract $\Duk$ from the set of 10,086 patients enrolled in the British Prostate Cancer UK program~\cite{Cutract2019}. These patients are characterized by the same features for both $\Duk$ and $\Dusa$. However, the datasets $\Duk$ and $\Dusa$ differ by a covariate shift: patients from $\Duk$ are in general older and at earlier clinical stages.  

When comparing the two populations in terms of the model, a first interesting question to ask is whether the covariate shift between $\Dusa$ and $\Duk$ affects the model representation. To explore this question, we take a first corpus of American patients $\Cusa~\subset~\Dusa$. If there is indeed a difference in terms of the latent representations, we expect the representations of test examples from $\Duk$ to be less closely approximated by their decomposition with respect to $\Cusa$. If this is true, the corpus residuals associated to examples of $\Duk$ will typically larger than the ones associated to $\Dusa$. To evaluate this quantitatively, we consider a mixed set of test examples $\T$ sampled from both $\Duk$ and $\Dusa$: $\T \subset \Duk \sqcup \Dusa$. We sample 100 examples from both sources: $\mid \T \cap \Duk \vert = \vert \T \cap \Dusa \vert = 100$. We then approximate the latent representation of each example $\h \in \g (\T)$ and compute the associated corpus residual $r_{\Cusa}(\h)$. We sort the test examples from $\T$ by decreasing order of corpus residual and we use this sorted list to see if we can detect the examples from $\Duk$. We use previous baselines for comparison, results are shown in Figure~\ref{fig:prostate_outlier}.

\begin{wrapfigure}{r}{.45\textwidth}
\vspace{-.5cm}
  \begin{center}
  \includegraphics[width=0.45\textwidth]{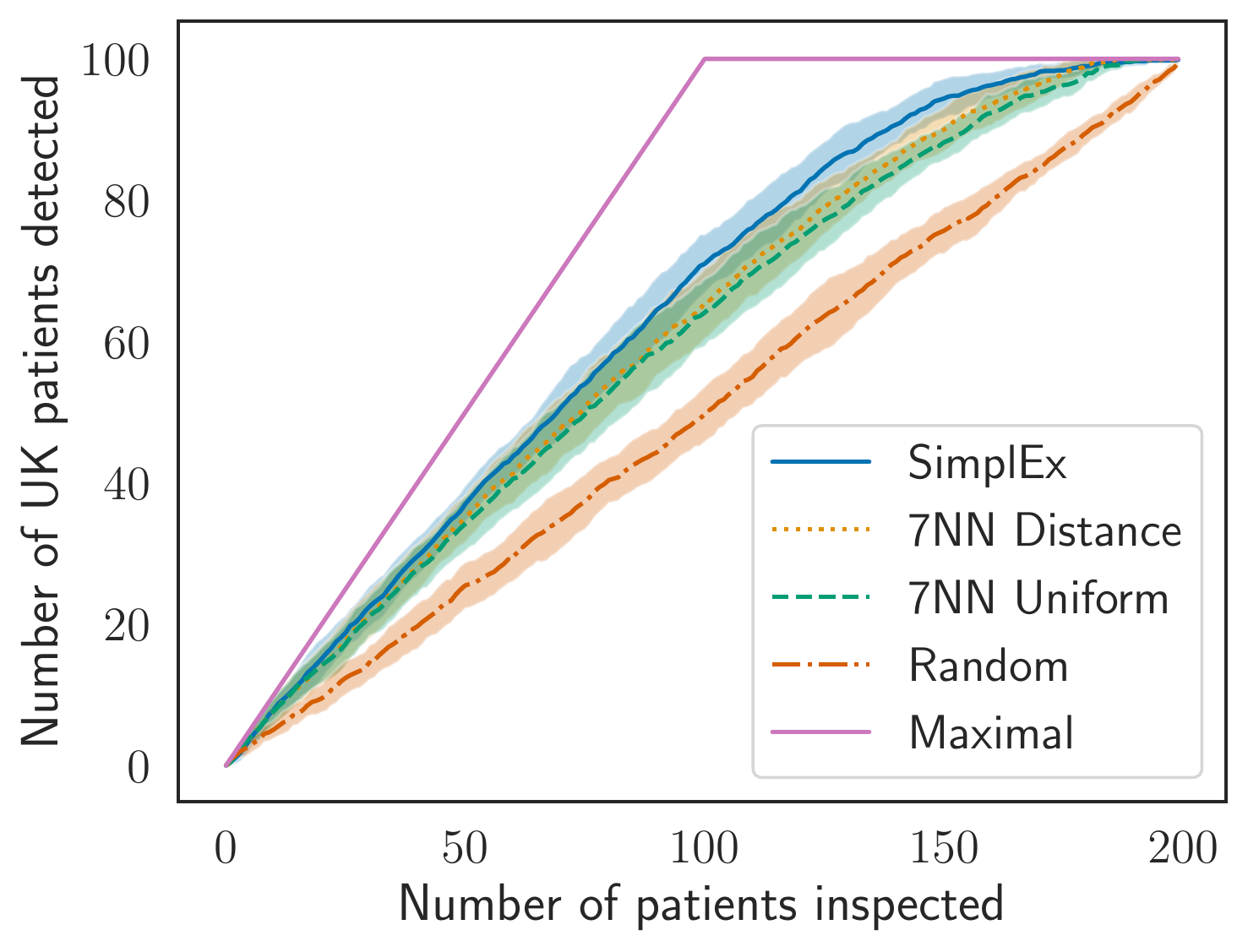}
  \end{center}
  \vspace{-0.3cm}
  \caption{Detecting UK patients (avg.$\pm$std.).}
  \label{fig:prostate_outlier}
  \vspace{-0.3cm}
\end{wrapfigure}
Several things can be deduced from this experiment. (1)~The results strongly suggest that the difference between the two datasets $\Dusa$ and $\Duk$ is reflected in their latent representations. (2)~The corpus residuals from SimplEx offer the most reliable way to detect examples that are different from the corpus examples $\Cusa$. None of the methods matches the maximal baseline since some examples of $\Dusa$ resemble examples from $\Duk$. (3)~When the corpus examples are representative of the training set, as it is the case in the experiment, our approach based on SimplEx provides a systematic way to detect test examples that have representations that are different from the ones produced at training time. A doctor should be more sceptical with respect to model predictions associated to larger residual with respect to $\Cusa$ as these arise from an extrapolation region of the latent space.

Let us now make the case more concrete. Suppose that an American and a British doctor use the above risk model to predict the outcome for their patients. Each doctor wants to decompose the predictions of the model in terms of patients they know. Hence, the American doctor selects a corpus of American patients $\Cusa \subset \Dusa$ and the British doctor selects a corpus of British patients $\Cuk \subset \Duk$. Both corpora have the same size $C_{\text{USA}} = C_{\text{UK}} = 1,000$. We suppose that the doctors know the model prediction and the true outcome for each patient in their corpus. Both doctors are sceptical about the risk model and want to use SimplEx to decide when it can be trusted. This leads them to a natural question: is it possible to anticipate misclassification with the help of SimplEx?

In Figure~\ref{fig:misclassified1} \& \ref{fig:misclassified2}, we provide two typical examples of misclassified British patients from $\Duk \setminus \Cuk$ together with their decomposition in terms of the two corpora $\Cusa$ and $\Cuk$. These two examples exhibit two qualitatively different situations. In Figure~\ref{fig:misclassified1}, both the American and the British doctors make the same observation: the model relates the test patient to corpus patients that are mostly misclassified by the model. With the help of SimplEx, both doctors will rightfully be sceptical with respect to the model's prediction.

In Figure~\ref{fig:misclassified2}, something even more interesting occurs: the two corpus decompositions suggest different conclusions. In the American doctor's perspective, the prediction for this patient appears perfectly coherent as all patients in the corpus decomposition have very similar features and all of them are rightfully classified. On the other hand, the British doctor will reach the opposite conclusion as the most relevant corpus patient is misclassified by the model. In this case, we have a perfect illustration of the limitation of the transfer of a risk model from one site (America) to another (United Kingdom): similar patients from different sites can have different outcomes. In both cases, since the test patient is British, only the decomposition in terms of $\Cuk$ really matters. In both cases, the British doctor could have anticipated the misclassification of each patient with SimplEx. 

\begin{figure}[h]
	\centering
	\includegraphics[width=.8\textwidth]{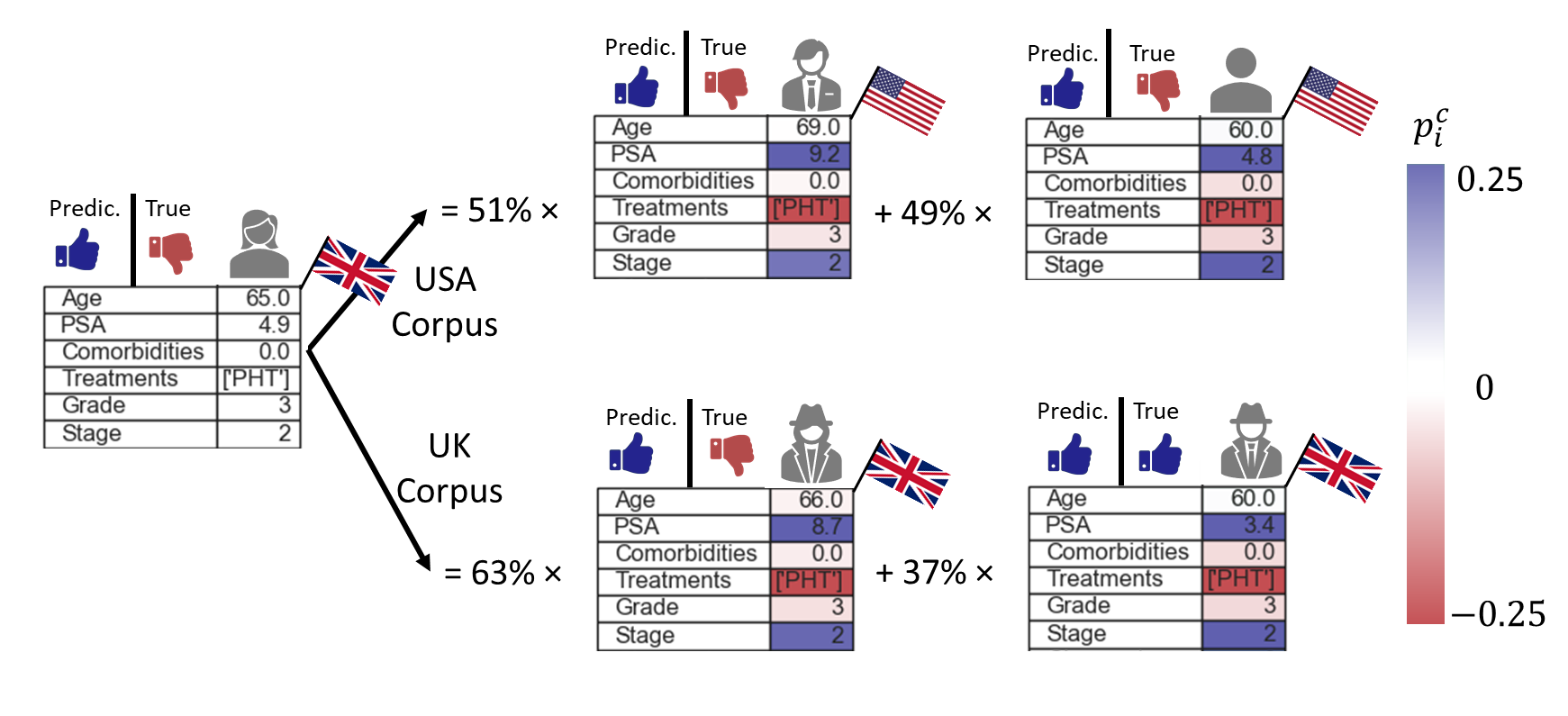}
	\caption{A first misclassified patient.}
	\vspace{-0.3cm}
	\label{fig:misclassified1}
\end{figure}

\begin{figure}[h]
  \centering
  \includegraphics[width=.9\textwidth]{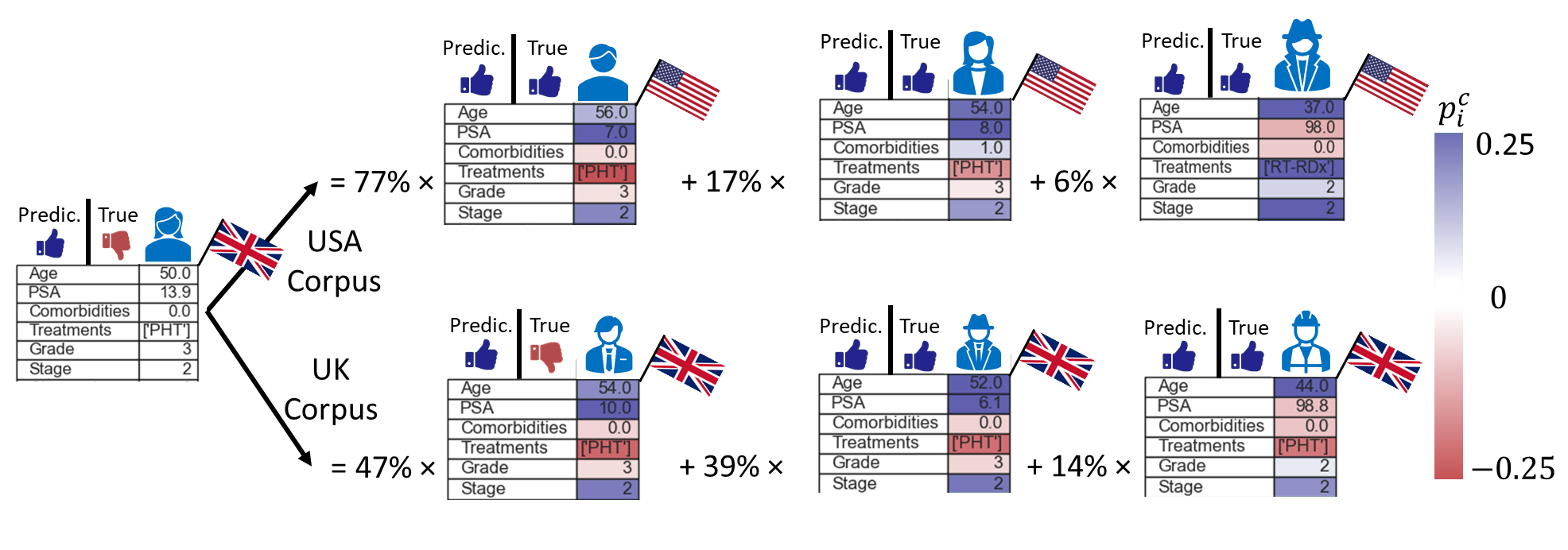}
  \caption{A second misclassified patient.}
  \label{fig:misclassified2}
\end{figure}

\section{Discussion}
We have introduced SimplEx, a method that decomposes the model representations at inference time in terms of a corpus. Through several experiments, we have demonstrated that these decompositions are accurate and can easily be personalized to the user. Finally, by introducing Integrated Jacobians, we have brought these explanations to the feature level. 
 
We believe that our bridge between feature and example-based explainability opens up many avenues for the future. A first interesting extension would be to investigate how SimplEx can be used to understand latent representations involved in unsupervised learning. For instance, SimplEx could be used to study the interpretability of \emph{self-expressive} latent representations learned by autoencoders~\cite{Ji2017}. A second interesting possibility would be to design a rigorous scheme to select the optimal corpus for a given model and dataset. Finally, a formulation where we allow the corpus to vary on the basis of observations would be particularly interesting for online learning.

\FloatBarrier
\newpage
\section*{Acknowledgement}
The authors are grateful to Alicia Curth, Krzysztof Kacprzyk, Boris van Breugel, Yao Zhang and the 4 anonymous NeurIPS 2021 reviewers for their useful comments on an earlier version of the manuscript. Jonathan Crabbé would like to thank Bogdan Cebere for replicating the experiments in this manuscript. Jonathan Crabbé would like to acknowledge Bilyana Tomova for many insightful discussions and her constant support. Jonathan Crabbé is funded by Aviva. Fergus Imrie is supported by by the National Science Foundation (NSF), grant number 1722516. Zhaozhi Qian and Mihaela van der Schaar are supported by the Office of Naval Research (ONR), NSF 1722516.  
\nocite{*}
\bibliographystyle{unsrt}
\bibliography{simplex}

\begin{thebibliography}{10}

\bibitem{Lipton2016}
Zachary~C. Lipton.
\newblock {The Mythos of Model Interpretability}.
\newblock {\em Communications of the ACM}, 61(10):35--43, 2016.

\bibitem{Ching2018}
Travers Ching, Daniel~S. Himmelstein, Brett~K. Beaulieu-Jones, Alexandr~A.
  Kalinin, Brian~T. Do, Gregory~P. Way, Enrico Ferrero, Paul-Michael Agapow,
  Michael Zietz, Michael~M. Hoffman, Wei Xie, Gail~L. Rosen, Benjamin~J.
  Lengerich, Johnny Israeli, Jack Lanchantin, Stephen Woloszynek, Anne~E.
  Carpenter, Avanti Shrikumar, Jinbo Xu, Evan~M. Cofer, Christopher~A.
  Lavender, Srinivas~C. Turaga, Amr~M. Alexandari, Zhiyong Lu, David~J. Harris,
  Dave DeCaprio, Yanjun Qi, Anshul Kundaje, Yifan Peng, Laura~K. Wiley, Marwin
  H.~S. Segler, Simina~M. Boca, S.~Joshua Swamidass, Austin Huang, Anthony
  Gitter, and Casey~S. Greene.
\newblock {Opportunities and obstacles for deep learning in biology and
  medicine}.
\newblock {\em Journal of The Royal Society Interface}, 15(141):20170387, 2018.

\bibitem{Tjoa2020}
Erico Tjoa and Cuntai Guan.
\newblock {A Survey on Explainable Artificial Intelligence (XAI): Toward
  Medical XAI}.
\newblock {\em IEEE Transactions on Neural Networks and Learning Systems},
  pages 1--21, 2020.

\bibitem{BarredoArrieta2020}
Alejandro {Barredo Arrieta}, Natalia D{\'{i}}az-Rodr{\'{i}}guez, Javier {Del
  Ser}, Adrien Bennetot, Siham Tabik, Alberto Barbado, Salvador Garcia, Sergio
  Gil-Lopez, Daniel Molina, Richard Benjamins, Raja Chatila, and Francisco
  Herrera.
\newblock {Explainable Artificial Intelligence (XAI): Concepts, taxonomies,
  opportunities and challenges toward responsible AI}.
\newblock {\em Information Fusion}, 58:82--115, 2020.

\bibitem{Das2020}
Arun Das and Paul Rad.
\newblock {Opportunities and Challenges in Explainable Artificial Intelligence
  (XAI): A Survey}.
\newblock {\em arXiv}, 2020.

\bibitem{Rai2019}
Arun Rai.
\newblock {Explainable AI: from black box to glass box}.
\newblock {\em Journal of the Academy of Marketing Science 2019 48:1},
  48(1):137--141, 2019.

\bibitem{Shapley1953}
Lloyd Shapley.
\newblock {A value for n-person games}.
\newblock {\em Contributions to the Theory of Games}, 2(28):307--317, 1953.

\bibitem{Datta2016}
Anupam Datta, Shayak Sen, and Yair Zick.
\newblock {Algorithmic Transparency via Quantitative Input Influence: Theory
  and Experiments with Learning Systems}.
\newblock In {\em Proceedings - 2016 IEEE Symposium on Security and Privacy, SP
  2016}, pages 598--617. Institute of Electrical and Electronics Engineers
  Inc., 2016.

\bibitem{Lundberg2017}
Scott Lundberg and Su-In Lee.
\newblock {A Unified Approach to Interpreting Model Predictions}.
\newblock {\em Advances in Neural Information Processing Systems}, pages
  4766--4775, 2017.

\bibitem{Ribeiro2016}
Marco~Tulio Ribeiro, Sameer Singh, and Carlos Guestrin.
\newblock {"Why should i trust you?" Explaining the predictions of any
  classifier}.
\newblock In {\em Proceedings of the ACM SIGKDD International Conference on
  Knowledge Discovery and Data Mining}, pages 1135--1144. Association for
  Computing Machinery, 2016.

\bibitem{Sundararajan2017}
Mukund Sundararajan, Ankur Taly, and Qiqi Yan.
\newblock {Axiomatic Attribution for Deep Networks}.
\newblock {\em 34th International Conference on Machine Learning, ICML 2017},
  7:5109--5118, 2017.

\bibitem{Dhurandhar2018}
Amit Dhurandhar, Pin-Yu Chen, Ronny Luss, Chun-Chen Tu, Paishun Ting,
  Karthikeyan Shanmugam, and Payel Das.
\newblock {Explanations based on the Missing: Towards Contrastive Explanations
  with Pertinent Negatives}.
\newblock {\em Advances in Neural Information Processing Systems}, pages
  592--603, 2018.

\bibitem{Fong2017}
Ruth~C. Fong and Andrea Vedaldi.
\newblock {Interpretable Explanations of Black Boxes by Meaningful
  Perturbation}.
\newblock {\em Proceedings of the IEEE International Conference on Computer
  Vision}, pages 3449--3457, 2017.

\bibitem{Fong2019}
Ruth Fong, Mandela Patrick, and Andrea Vedaldi.
\newblock {Understanding deep networks via extremal perturbations and smooth
  masks}.
\newblock {\em Proceedings of the IEEE International Conference on Computer
  Vision}, pages 2950--2958, 2019.

\bibitem{Crabbe2021}
Jonathan Crabb{\'{e}} and Mihaela van~der Schaar.
\newblock {Explaining Time Series Predictions with Dynamic Masks}.
\newblock {\em International Conference on Machine Learning}, 38, 2021.

\bibitem{Nguyen2021}
Giang Nguyen, Daeyoung Kim, and Anh Nguyen.
\newblock {The effectiveness of feature attribution methods and its correlation
  with automatic evaluation scores}.
\newblock {\em Advances in Neural Information Processing Systems}, 35, 2021.

\bibitem{Caruana1999}
R.~Caruana, H.~Kangarloo, J.~D. Dionisio, U.~Sinha, and D.~Johnson.
\newblock {Case-based explanation of non-case-based learning methods.}
\newblock {\em AMIA Symposium}, pages 212--215, 1999.

\bibitem{Bichindaritz2006}
Isabelle Bichindaritz and Cindy Marling.
\newblock {Case-based reasoning in the health sciences: What's next?}
\newblock In {\em Artificial Intelligence in Medicine}, volume~36, pages
  127--135. Elsevier, 2006.

\bibitem{Keane2019}
Mark~T. Keane and Eoin~M. Kenny.
\newblock {How Case-Based Reasoning Explains Neural Networks: A Theoretical
  Analysis of XAI Using Post-Hoc Explanation-by-Example from a Survey of
  ANN-CBR Twin-Systems}.
\newblock In {\em Lecture Notes in Computer Science (including subseries
  Lecture Notes in Artificial Intelligence and Lecture Notes in
  Bioinformatics)}, volume 11680 LNAI, pages 155--171. Springer Verlag, 2019.

\bibitem{Kim2015}
Been Kim, Cynthia Rudin, and Julie Shah.
\newblock {The Bayesian Case Model: A Generative Approach for Case-Based
  Reasoning and Prototype Classification}.
\newblock {\em Advances in Neural Information Processing Systems},
  3:1952--1960, 2015.

\bibitem{Kim2016}
Been Kim, Rajiv Khanna, and Oluwasanmi Koyejo.
\newblock {Examples are not Enough, Learn to Criticize! Criticism for
  Interpretability}.
\newblock {\em Advances in Neural Information Processing Systems}, 29, 2016.

\bibitem{Gurumoorthy2017}
Karthik~S. Gurumoorthy, Amit Dhurandhar, Guillermo Cecchi, and Charu Aggarwal.
\newblock {Efficient Data Representation by Selecting Prototypes with
  Importance Weights}.
\newblock {\em Proceedings - IEEE International Conference on Data Mining,
  ICDM}, pages 260--269, 2017.

\bibitem{Cook1982}
R.~Dennis Cook and Sanford Weisenberg.
\newblock {\em {Residuals and influence in regression}}.
\newblock New York: Chapman and Hall, 1982.

\bibitem{Koh2017}
Pang~Wei Koh and Percy Liang.
\newblock {Understanding Black-box Predictions via Influence Functions}.
\newblock {\em 34th International Conference on Machine Learning, ICML 2017},
  4:2976--2987, 2017.

\bibitem{Ghorbani2019}
Amirata Ghorbani and James Zou.
\newblock {Data Shapley: Equitable Valuation of Data for Machine Learning}.
\newblock {\em 36th International Conference on Machine Learning, ICML 2019},
  pages 4053--4065, 2019.

\bibitem{Ghorbani2020}
Amirata Ghorbani, Michael~P. Kim, and James Zou.
\newblock {A Distributional Framework for Data Valuation}.
\newblock {\em arXiv}, 2020.

\bibitem{Kim2017}
Been Kim, Martin Wattenberg, Justin Gilmer, Carrie Cai, James Wexler, Fernanda
  Viegas, and Rory Sayres.
\newblock {Interpretability Beyond Feature Attribution: Quantitative Testing
  with Concept Activation Vectors (TCAV)}.
\newblock {\em 35th International Conference on Machine Learning, ICML 2018},
  6:4186--4195, 2017.

\bibitem{Papernot2018}
Nicolas Papernot and Patrick McDaniel.
\newblock {Deep k-Nearest Neighbors: Towards Confident, Interpretable and
  Robust Deep Learning}.
\newblock {\em arXiv}, 2018.

\bibitem{Yeh2018}
Chih-Kuan Yeh, Joon~Sik Kim, Ian E.~H. Yen, and Pradeep Ravikumar.
\newblock {Representer Point Selection for Explaining Deep Neural Networks}.
\newblock {\em Advances in Neural Information Processing Systems}, pages
  9291--9301, 2018.

\bibitem{Keane2019b}
Mark~T. Keane and Eoin~M. Kenny.
\newblock {The Twin-System Approach as One Generic Solution for XAI: An
  Overview of ANN-CBR Twins for Explaining Deep Learning}.
\newblock {\em arXiv}, 2019.

\bibitem{Seer2019}
Surveillance Research~Program National Cancer~Institute, DCCPS.
\newblock Surveillance, epidemiology, and end results (seer) program.
\newblock www.seer.cancer.gov, 2019.

\bibitem{Deng2012}
Li~Deng.
\newblock The mnist database of handwritten digit images for machine learning
  research.
\newblock {\em IEEE Signal Processing Magazine}, 29(6):141--142, 2012.

\bibitem{Wu2021}
Eric Wu, Kevin Wu, Roxana Daneshjou, David Ouyang, Daniel~E Ho, and James Zou.
\newblock {How medical AI devices are evaluated: limitations and
  recommendations from an analysis of FDA approvals}.
\newblock {\em Nature Medicine}, 2021.

\bibitem{Cutract2019}
Prostate~Cancer UK.
\newblock Cutract.
\newblock www.prostatecanceruk.org, 2019.

\bibitem{Ji2017}
Pan Ji, Tong Zhang, Hongdong Li, Mathieu Salzmann, and Ian Reid.
\newblock {Deep Subspace Clustering Networks}.
\newblock {\em Advances in Neural Information Processing Systems}, pages
  24--33, 2017.

\bibitem{Anirudh2017}
Rushil Anirudh, Jayaraman~J. Thiagarajan, Rahul Sridhar, and Peer-Timo Bremer.
\newblock {MARGIN: Uncovering Deep Neural Networks using Graph Signal
  Analysis}.
\newblock {\em arXiv}, 2017.

\bibitem{Cohen2017}
Gregory Cohen, Saeed Afshar, Jonathan Tapson, and Andr{\'{e}} van Schaik.
\newblock {EMNIST: an extension of MNIST to handwritten letters}.
\newblock {\em arXiv}, 2017.

\bibitem{Abadie2020}
Alberto Abadie.
\newblock {Using Synthetic Controls: Feasibility, Data Requirements, and
  Methodological Aspects}.
\newblock {\em Journal of Economic Literature}, 2020.

\bibitem{Abadie2010}
Alberto Abadie, Alexis Diamond, and Jens Hainmueller.
\newblock {Synthetic control methods for comparative case studies: Estimating
  the effect of California's Tobacco control program}.
\newblock {\em Journal of the American Statistical Association},
  105(490):493--505, 2010.

\bibitem{Athey2017}
Susan Athey, Mohsen Bayati, Nikolay Doudchenko, Guido Imbens, and Khashayar
  Khosravi.
\newblock {Matrix Completion Methods for Causal Panel Data Models}.
\newblock {\em arXiv}, 2017.

\bibitem{Amjad2018}
Muhammad Amjad, Devavrat Shah, and Dennis Shen.
\newblock {Robust Synthetic Control}.
\newblock {\em Journal of Machine Learning Research}, 19(22):1--51, 2018.

\bibitem{Gordetsky2016}
Jennifer Gordetsky and Jonathan Epstein.
\newblock {Grading of prostatic adenocarcinoma: current state and prognostic
  implications}.
\newblock {\em Diagnostic Pathology}, 11(1), 2016.

\end{thebibliography}

\newpage
\appendix

\section{Supplement for Mathematical Formulation} \label{sec:problem_sup}
In this supplementary section, we give more details on the mathematical aspects underlying SimplEx.
\subsection{Precision of the corpus approximation in output space} \label{subsec:precision_output}

If the corpus representation of $\h \in \H$ has a residual $r_{\C}(\h)$, Assumption~\ref{assumption-linear} controls the error between the black-box prediction for the test example $\f(\x) = \l(\h)$ and and its corpus representation $\l (\hat{\h})$. 

\begin{proposition}[Precision in output space]
Consider a latent representation $\h$ with corpus residual $r_{\C}(\h)$. If Assumption~\ref{assumption-linear} holds, this implies that the corpus prediction $\l (\hat{\h})$ approximates $\l(\h)$ with a precision controlled by the corpus residual:
\begin{align*}
\parallel \l(\hat{\h}) - \l( \h) \parallel_{\Y} \hspace{0.5cm} \leq \hspace{0.5cm} \ \parallel \l \parallel_{\text{op}} \ \cdot \ r_{\C} (\h),
\end{align*} 
where $\parallel \cdot \parallel_{\Y}$ is a norm on $\Y$ and $\parallel \l \parallel_{\text{op}} = \inf \left\{ \lambda \in \R^+ : \ \parallel\l(\tilde{\h})\parallel_{\Y} \ \leq \ \lambda  \parallel \tilde{\h}\parallel_{\H} \ \forall \tilde{\h} \in \H \right\}$ is the usual operator norm.
\end{proposition}
\begin{proof}
The proof is immediate:
\begin{align*}
\norm{\l(\hat{\h}) - \l( \h)}{\Y} \hspace{.5cm} &= \hspace{.5cm} \parallel \l(\hat{\h} - \h) \parallel_{\Y}  \\
 &\leq \hspace{.5cm} \norm{\l}{\text{op}} \cdot \norm{\hat{\h} - \h}{\H} \\ 
 &= \hspace{.5cm} \norm{\l}{\text{op}} \cdot \ r_{\C}(\h),
\end{align*}
where we have successively used the linearity of $\l$, the definition of the operator norm $\norm{\cdot}{\text{op}}$ and Definition~\ref{def-corpus_residual}.
\end{proof}

\subsection{Uniqueness of corpus decomposition}

As we have mentioned in the main paper, the corpus decomposition is not always unique. To illustrate, we consider the following corpus representation: $\g( \C ) = \left\{ \h^1 , \h^2 , \h^3 = 0.5 \cdot \h^1 + 0.5 \cdot \h^2 \right\}$. Consider the following vector in the corpus hull: $\h = 0.75 \cdot \h^1 + 0.25 \cdot \h^2$. We note that this vector can also be written as $\h = 0.5 \cdot \h^1 + 0.5 \cdot \h^3 $. In other words, the vector $\h \in \g(\C)$ admits more than one corpus decomposition. This is not a surprise for the attentive reader: by paying a closer look to $\g(\C)$, we note that $\h^3$ is somewhat redundant as it is itself a combination of $\h^1$ and $\h^2$. The multiplicity of the corpus decomposition results from a redundancy in the corpus representation.

To make this reasoning more general, we need to revisit some classic concepts of convex analysis. To establish a sufficient condition that guarantees the uniqueness of corpus decompositions, we recall the definition of affine independence.

\begin{definition}[Affine independence]
The vectors $\{ \h^c \mid c \in [C] \} \subset \R^d$ are \emph{affinely independent} if
\begin{align*}
\sum_{c=1}^C \lambda^c \h^c = 0 \ \wedge \ \sum_{c=1}^C \lambda^c = 0 \ \Longrightarrow \lambda^c = 0 \ \forall \ c \in [C] 
\end{align*}
\end{definition}
If a set of vectors is not affinely independent, it means that one of the vectors can be written as an affine combination of the others. This is precisely what we called a redundancy in the previous paragraph. We now adapt a well-known result of convex analysis to our formalism:

\begin{proposition}[Uniqueness of corpus decomposition]
If the corpus representation $\g(\C) = \{ \h^c \mid c \in [C] \} $ is a set of affinely independent vectors, then every vector in the corpus hull $\h \in \CH(\C)$ admits one unique corpus decomposition.
\end{proposition}
\begin{proof}
The existence of a decomposition is a trivial consequence of the definition of $\CH(\C)$. We prove the uniqueness of the decomposition by contradiction. Let us assume that a vector $\h \in \CH(\C)$ admits two distinct corpus decompositions:
\begin{align*}
\h = \sum_{c=1}^C w^c \h^c = \sum_{c=1}^C \tilde{w}^c \h^c \\
\end{align*}
where $w^c , \tilde{w}^c \geq 0$ for all $c \in [C]$, $\sum_{c=1}^C w^c = \sum_{c=1}^C \tilde{w}^c = 1$ and $(w^c)_{c=1}^C \neq (\tilde{w}^c)_{c=1}^C$. It follows that:
\begin{align*}
\sum_{c=1}^C \underbrace{(w^c - \tilde{w}^c )}_{\equiv \lambda^c} \h^c = 0
\end{align*}
But $\sum_{c=1}^C \lambda^c = \sum_{c=1}^C (w^c - \tilde{w}^c ) = 1 - 1 = 0$. It follows that $\g(\C)$ is not affinely independent, a contradiction.
\end{proof}
This shows that affine independence provides a sufficient condition to ensure the uniqueness of corpus decompositions. If one wants to produce such a corpus, a possibility is to gradually add new examples in the corpus by checking that the latent representation of each new example is not an affine combination of the previous latent representations. Clearly, the number of examples in such a corpus cannot exceed $d_H + 1$.

\subsection{Integrated Jacobian and Integrated Gradients} \label{subsec:integrated_gradient}

Integrated Gradients is a notorious method used to discuss feature saliency~\cite{Sundararajan2017}. It uses a black-box output to attribute a saliency score to each feature. In the original paper, the output space $\Y$ is assumed to be one-dimensional: $d_Y = 1$. We shall therefore relax the bold notation that we have used for the outputs so far. In this way, the black-box is denoted $f$ and the latent-to-output map is denoted $l$. Although the original paper makes no mention of corpus decompositions, it is straightforward to adapt the definition of Integrated Gradients to our set-up:
\begin{definition}[Integrated Gradient] The \emph{Integrated Gradient} between a baseline $\x^0$ an a corpus example $\x^c \in \X$ associated to feature $i \in [d_X]$ is
\begin{align*}
\IG_i^c = \left( x^c_i - x^0_i \right) \int_0^1 \frac{\partial f}{\partial x_i} \bigg|_{\linec(t)} \ dt \hspace{0.5cm} \in \R,
\end{align*}
where $\linec (t) \equiv \x^0 + t \cdot \left( \x^c - \x^0 \right)$ for $t \in [0,1]$. 
\end{definition}
In the main paper, we have introduced Integrated Jacobians: a latent space generalization of Integrated Gradients. We use the word generalization for a reason: the Integrated Gradient can be deduced from the Integrated Jacobian but not the opposite\footnote{Unless in the degenerate case where $d_H = 1$. However, this case is of little interest as it describes a situation where $\Y$ and $\H$ are isomorphic, hence the distinction between output and latent space is fictional.}. We make the relationship between the two quantities explicit in the following proposition.
\begin{proposition}
The Integrated Gradient can be deduced from the Integrated Jacobian via
\begin{align*}
\IG^c_i = l \left( \j^c_i \right).
\end{align*}
\end{proposition}
\begin{proof}
We start from the definition of the Integrated Gradient:
\begin{align*}
\IG_i^c \hspace{.5cm} &= \hspace{.5cm} \left( x^c_i - x^0_i \right) \int_0^1 \frac{\partial f}{\partial x_i} \bigg|_{\linec(t)} \ dt \\
&= \hspace{.5cm} \left( x^c_i - x^0_i \right) \int_0^1 \partderiv{(l \circ \g)}{x_i} \bigg|_{\linec(t)} \ dt \\
&= \hspace{.5cm} \left( x^c_i - x^0_i \right) \int_0^1 l \left( \partderiv{\g}{x_i} \bigg|_{\linec(t)} \right) \ dt \\
&= \hspace{.5cm} \left( x^c_i - x^0_i \right) l \left( \int_0^1  \partderiv{\g}{x_i} \bigg|_{\linec(t)}  \ dt \right) \\
&= \hspace{.5cm} l \left( \left( x^c_i - x^0_i \right)  \int_0^1  \partderiv{\g}{x_i} \bigg|_{\linec(t)}  \ dt \right) \\
&= \hspace{.5cm} l \left( \j^c_i \right),
\end{align*}
where we have successively used: Assumption~\ref{assumption-linear}, the linearity of the partial derivative, the linearity of the integration operator, the linearity of $l$ and the definition of Integrated Jacobians.
\end{proof}
Note that Integrated Jacobians allow us to push our understanding of the black-box beyond the output. There is very little reason to expect a one dimensional output to capture the model complexity. As we have argued in the introduction, our paper pursues the more challenging ambition of gaining a deeper understanding of the black-box latent space. 

\subsection{Properties of Integrated Jacobians} \label{subsec:proof_prop}
We give a proof for the proposition appearing in the main paper.
\begin{manualproposition}{2.1}[Properties of Integrated Jacobians]
Consider a baseline $(\x^0, \h^0 = \g(\x^0))$ and a test example together with their latent representation $(\x, \h = \g(\x)) \in \X \times \H$. If the shift $\h - \h^0$ admits a decomposition~\eqref{equ-shift_corpus_decomposition}, the following properties hold.\\
\begin{align*}
(A):~\sum_{c=1}^C \sum_{i=1}^{d_X} w^c \textbf{j}^c_i = \h - \h^0 \hspace{1cm} (B):~\sum_{c=1}^C \sum_{i=1}^{d_X} w^c p^c_i = 1.
\end{align*} 
\end{manualproposition}
\begin{proof}
Let us begin by proving (A). By using the chain rule for a given corpus example $c \in [C]$, we write explicitly the derivative of the curve $\curvec$ with respect to its parameter $t \in (0,1)$:
\begin{align*}
\deriv{\left( \curvec \right)}{t}\bigg|_{t} \hspace{.5cm}  &= \hspace{.5cm} \sum_{i=1}^{d_X} \partderiv{\g}{x_i}\bigg|_{\linec(t)} \deriv{\gamma^c_i}{t}\bigg|_{t} \\
&= \hspace{.5cm} \sum_{i=1}^{d_X} \partderiv{\g}{x_i}\bigg|_{\linec(t)} (x^c_i - x^0_i),
\end{align*}
where we used $\gamma^c_i(t) = x^0_i + t \cdot (x^c_i - x^0_i)$ to obtain the second equality. We use this equation to rewrite the sum of the Integrated Jacobians for this corpus example $c$:
\begin{align*}
\sum_{i=1}^{d_X} \j^c_i \hspace{.5cm} &= \hspace{.5cm} \sum_{i=1}^{d_X} \int_0^1 \partderiv{\g}{x_i}\bigg|_{\linec(t)} (x^c_i - x^0_i) \ dt \\
&= \hspace{.5cm} \int_0^1 \sum_{i=1}^{d_X}  \partderiv{\g}{x_i}\bigg|_{\linec(t)} (x^c_i - x^0_i) \ dt \\
&= \hspace{.5cm} \int_0^1 \deriv{\left( \curvec \right)}{t}\bigg|_{t} \ dt \\
&= \hspace{.5cm} \curvec(1) - \curvec(0) \\
&= \hspace{.5cm} \h^c - \h^0, \\
\end{align*}
where we have successively used: the linearity of integration, the explicit expression for the curve derivative, the fundamental theorem of calculus and the definition of the curve $\curvec$. We are now ready to derive (A):
\begin{align*}
\sum_{c=1}^C \sum_{i=1}^{d_X} w^c \textbf{j}^c_i \hspace{.5cm} &= \hspace{.5cm} \sum_{c=1}^C w^c \left( \h^c - \h^0 \right)\\
&= \hspace{.5cm} \h - \h^0,
\end{align*}
where we have successively used the exact expression for the sum of Integrated Jacobians associated to corpus example $c$ and the definition of the corpus decomposition of $\h$. We are done with (A), let us now prove (B). We simply project both members of (A) on the overall shift $\h - \h^0$. Projecting the left-hand side of (A) yields:
\begin{align*}
\proj_{\h - \h^0} \left( \sum_{c=1}^C \sum_{i=1}^{d_X} w^c \textbf{j}^c_i \right) \hspace{.5cm} &= \hspace{.5cm} \sum_{c=1}^C \sum_{i=1}^{d_X} w^c \underbrace{\proj_{\h - \h^0}  \left(  \textbf{j}^c_i \right)}_{p^c_i},
\end{align*} 
where we used the linearity of the projection operator. Projecting the right-hand side of (A) yields:
\begin{align*}
\proj_{\h - \h^0} (\h - \h^0)  = \frac{\langle \ \h - \h^0 \ , \ \h - \h^0 \ \rangle}{\langle \ \h - \h^0 \ , \ \h - \h^0 \ \rangle} = 1.
\end{align*}
By equating the projected version of both members of (A), we deduce (B).
\end{proof}

\subsection{Pseudocode for SimplEx} \label{subsec:pseudocode}

We give the pseudocode for the two modules underlying SimplEx: the corpus decomposition (Algorithm~\ref{algo:corpus_decomposition}) and the evaluation of projected Jacobians (Algorithm~\ref{algo:projected_jacobian}).

\begin{algorithm}[H] \label{algo:corpus_decomposition}
\SetAlgoLined
\KwIn{Test latent representation $\h$ ; Corpus representation $\left\{ \h^c \mid c \in [C] \right\}$}
\KwResult{Weights of the corpus decomposition $\w \in [0,1]^C$ ; Corpus residual $r_{\C}(\h)$}
Initialize pre-weights: $\tilde{\w} \leftarrow \boldsymbol{0} $\;
\While{optimizing}{
 Normalize pre-weights: $\w \leftarrow \textbf{softmax}\left[ \tilde{\w} \right]$\;
 Evaluate loss: $L \left[ \tilde{\w} \right] \leftarrow \sum_{i=1}^{d_H} \left( h_i - \sum_{c=1}^C w^c h^c_i \right)^2$ \;
 Update pre-weights: $\tilde{\w} \leftarrow \textbf{Adam step} \left( L \left[ \tilde{\w} \right] \right)$\;
 }
 Return normalized weights: $\w \leftarrow \textbf{softmax}\left[ \tilde{\w} \right]$\;
 Return corpus residual: $r_{\C}(\h) \leftarrow  \left[ \sum_{i=1}^{d_H} \left( h_i - \sum_{c=1}^C w^c h^c_i \right)^2 \right]^{1/2}$\;
\caption{SimplEx: Corpus Decomposition}
\end{algorithm}

Where we used a vector notation for the pre-weights and the weights: $\w = \left(w^c \right)_{c=1}^C$.
For the Adam optimizer, we use the default hyperparameters in the Pytorch implementation: $\text{lr} = 10^{-3}$ ; $\beta_1 = .9$ ; $\beta_2 = .999$ ; $\text{eps} = 10^{-8} $. Note that this algorithm is a standard optimization loop for a convex problem where the normalization of the weights is ensured by using a softmax.

When the size of the corpus elements to contribute has to be limited to $K$, we use a similar strategy as the one used to produce extremal perturbations~\cite{Fong2019, Crabbe2021}. This consists in adding the following $L^1$ term to the optimized loss $L$:
\begin{align*}
L_{\text{reg}} \left[ \tilde{\w} \right] = \sum_{d=1}^{C-K} \left|  \text{vecsort}^d \left[ \tilde{\w}  \right] \right|,
\end{align*}
where vecsort is a permutation operator that sorts the components of a vector in ascending order.
The notation $\text{vecsort}^d$ refers to the $d^{th}$ component of the sorted vector.
This regularization term will impose sparsity for the $C-K$ smallest weights of the corpus decomposition. As a result, the optimal corpus decomposition only involves $K$ non-vanishing weights. We now focus on the evaluation of the Projected Jacobian.

\begin{algorithm}[H] \label{algo:projected_jacobian}
\SetAlgoLined
\KwIn{Test input $\x$ ; Test representation $\h$ ; Corpus $\left\{ \x^c \mid c \in [C] \right\}$ ; Corpus representation $\left\{ \h^c \mid c \in [C] \right\}$ ; Baseline input $\x^0$ ; Baseline representation $\h^0$ ; Black-box latent map $\g$ ; Number of bins $N_b \in \N^*$}
\KwResult{Jacobian projections $\P = \left( p^c_i \right) \in \R^{C \times d_X}$}
Initialize the projection matrix: $\P = \left( p^c_i \right) \leftarrow \boldsymbol{0}$ \;
Form a matrix of corpus inputs: $\textbf{X}^C \leftarrow (x^c_i) \in \R^{C \times d_X} $\;
Form a matrix of baseline inputs: $\textbf{X}^0 \leftarrow (x^0_i) \in \R^{C \times d_X} $\;
\For{$n \in [N_b]$}{
Set the evaluation input: $\tilde{\textbf{X}} \leftarrow \textbf{X}^0 + \frac{n}{N_b} \left( \textbf{X}^C - \textbf{X}^0 \right)$ \;
Increment the Jacobian projections: $p^c_i \leftarrow p^c_i + \partderiv{\g}{x_i} \big|_{\tilde{\x}^c} \cdot \frac{\h - \h^0}{\norm{\h - \h^0}{2}^2} \hspace{.5cm} \forall (c,i) \ \in [C] \times [d_X] $ \;
}
Apply the appropriate pre-factor: $\P \leftarrow \frac{1}{N_b} \left( \textbf{X}^C - \textbf{X}^0 \right) \odot \P $ \; 
\caption{SimplEx: Projected Jacobian}
\end{algorithm}

This algorithm approximates the integral involved in the definition of the Projected Jacobian with a standard Riemann sum. Note that the definition of $\textbf{X}^0$ implies that the baseline vector $\x^0$ is broadcasted along the first dimension of the matrix. More explicitly, the components of this matrix are $X^0_{c,i} = x^0_i$ for $c \in [C]$ and $i \in [d_X]$. Also note that the projected Jacobians can be computed in parallel with packages such as Pytorch's autograd. We have used the notation $\odot$ to denote the conventional Hadarmard product. In our implementation, the number of bins $N_b$ is fixed at $200$, bigger $N_b$ don't significantly improve the precision of the Riemann sum.  

\subsection{Choice of a baseline for Integrated Jacobians}
Throughout our analysis of the corpus decomposition, we have assumed the existence of a baseline $\x_0 \in \X$. This baseline is crucial as it defines the starting point of the line $\linec$ that we use to compute the Jacobian quantities. What is a good choice for the baseline? The answer to this question depends on the domain. When this makes sense, we choose the baseline to be an instance that does not contain any information. A good example of this is the baseline that we use for MNIST: an image that is completely black $\x_0 = 0$. Sometimes, this absence of information is not well-defined. A good example is the prostate cancer experiment: it makes little sense to define a patient whose features contain no information. In this set-up, our baseline is a patient whose features are fixed to their average value in the training~\footnote{This average could also be computed with respect to the corpus itself.} set  $\x_0 = \vert \Dtrain \vert^{-1} \sum_{\x \in \Dtrain} \x$. In this way, a shift with respect to the baseline corresponds to a patient whose features differ from the population average.

\section{Supplement for Experiments}  \label{sec:experiments_sup}
In this section, we give more details on the experiments that we have conducted with SimplEx.
All our experiments have been performed on a machine with Intel(R) Core(TM) i5-8600K CPU @~3.60GHz [6 cores] and Nvidia GeForce RTX 2080 Ti GPU. Our implementation is done with Pytorch 1.8.1.

\subsection{Details for the corpus precision experiment} \label{subsec:precision_sup}

\textbf{Metrics} We give the explicit expression for the two metrics used in the experiment. By keeping the notation of Section~\ref{sec:experiment}, we assume that we are given a set of test samples $\T \subset \X$. For each test representation $\h \in \g (\T)$ and test output $\y \in \f (\T)$, we build build an approximation $\hat{\h}$ and $\hat{\y}$ with several methods. To evaluate the quality of these approximations, we use the following $R^2$ scores:
\begin{align*}
R^2_{\H} &= 1 - \frac{\sum_{\h \in \g(\T)}  \norm{\h - \hat{\h}}{2}^2}{\sum_{\h \in \g(\T)} \norm{\h - \bar{\h}}{2}^2 }  & \bar{\h} = \frac{1}{|\T|} \sum_{\h \in \g(\T)} \h \\ 
R^2_{\Y} &= 1 - \frac{\sum_{\y \in \f(\T)}  \norm{\y - \hat{\y}}{2}^2}{\sum_{\y \in \f(\T)} \norm{\y - \bar{\y}}{2}^2 } & \bar{\y} = \frac{1}{|\T|} \sum_{\y \in \f(\T)} \y .
\end{align*}
These $R^2$ scores compare the approximation method to a dummy approximator that approximates every representation and output with their test average. A negative value for $R^2$ indicates that the approximation method performs more poorly than the dummy approximator. Ideally, the $R^2$ score should be close to $1$.

\textbf{SEER Dataset} The SEER dataset is a private dataset consisting in 240,486 patients enrolled in the American SEER program~\cite{Seer2019}. All the patients from the SEER dataset have been de-identified. We consider the binary classification task of predicting cancer mortality for patients with prostate cancer. Each patient in the dataset is represented by the couple $(\x , \z)$, where $\x$ contains the patient features and $\z \in \{ 0 , 1 \}^2$ is a vector indicating the patient mortality. The features characterizing each patient are their age, PSA, Gleason score, clinical stage and which, if any, treatment they are receiving. These features are summarized in Table~\ref{tab:seer_features}. The original dataset is severely imbalanced as $93.8\%$ patients survive or have a mortality unrelated to cancer. We extract a balanced subset of 42,000 patients that we split into a training set $\Dtrain$ of 35,700 patients and a test set $\Dtest$ of 6,300 patients. We train a multilayer perceptron (MLP) for the mortality prediction task on $\Dtrain$. 

\begin{table}
\begin{tabularx}{\textwidth}{c X}
\toprule
Feature & Range \\ 
\midrule
Age & $60-73$ \\ 
PSA & $5-11$ \\ 
Comorbidities & $0, 1, 2, \geq 3$ \\ 
Treatment & Hormone Therapy (PHT), Radical Therapy - RDx (RT-RDx), \newline Radical Therapy -Sx (RT-Sx), CM \\ 
Grade & $1, 2, 3, 4, 5$ \\ 
Stage & $1, 2, 3, 4$ \\ 
Primary Gleason & $1, 2, 3, 4, 5$ \\ 
Secondary Gleason & $1, 2, 3, 4, 5$ \\ 
\bottomrule
\end{tabularx} 
\vspace{.5cm}
\caption{Features for the SEER Dataset.}
\label{tab:seer_features}
\end{table}

\textbf{MNIST Dataset} MNIST is a public dataset consisting in 70,000 MNIST images of handwritten digits~\cite{Deng2012}. We consider the multiclass classification task of identifying the digit represented on each image. Each instance in the dataset is represented by the couple $(\x , \z)$, where $\x$ contains the image itself and $\z \in \{ 0 , 1 \}^{10}$ is a vector indicating the true label for the image. The images are characterized by $28 \times 28$ pixels with one channel. The dataset is conventionally split into a training set $\Dtrain$ of 60,000 images and a test set $\Dtest$ of 10,000 images. We train a convolutional neural network (CNN) for the image classification task on $\Dtrain$. Yann LeCun and Corinna Cortes hold the copyright of MNIST dataset, which is a derivative work from original NIST datasets. MNIST dataset is made available under the terms of the Creative Commons Attribution-Share Alike 3.0 license.

\textbf{Prostate Cancer Model} The model that we use for mortality prediction on the SEER dataset is a simple MLP with two hidden layers. Its precise architecture is described in Table~\ref{tab:mortality_mlp}. The model is trained by minimizing the following loss:
\begin{align} \label{eq:classification_train_loss}
\mathcal{L}_{\text{train}}\left( \boldsymbol{\theta} \right) = \sum_{\left( \x , \z \right) \in \Dtrain } - \z \odot \textbf{log} \left[ \f_{\boldsymbol{\theta}} (\x) \right] + \lambda \cdot \norm{\boldsymbol{\theta}}{2},
\end{align}
where we have introduced a $L^2$ regularization, as required by the representer theorem. The regularization coefficient is chosen to be $\lambda = 10^{-5}$ (a bigger $\lambda$ significantly decreases the performance of the model on the testing set). We train this model with Adam (default Pytorch hyperparameters) for 5 epochs. Across the different runs, the accuracy of the resulting model on the test set ranges between $85-86 \%$. 
\begin{table}
\begin{tabularx}{\textwidth}{c c c c X}
\toprule
Layer & Input Dimension & Output Dimension & Activation & Remark \\ 
\midrule
Batch Norm & 3 & 3 &  & Only acts on Age, PSA and Comorbidities \\
Dense 1 & 26 & 200 & ReLU &  \\
Dropout & 200 & 200 &  &  \\
Dense 2 & 200 & 50 & ReLU &  \\
Dropout & 50 & 50 &  & Output: $\h = \g(\x)$  \\
Linear & 50 & 2 &  & Output: $\y = \l(\h) $  \\
Softmax & 2 & 2 &  & Output: $\p = \f(\x)$  \\
\bottomrule
\end{tabularx} 
\vspace{.5cm}
\caption{Mortality Prediction MLP for SEER.}
\label{tab:mortality_mlp}
\end{table}

\textbf{MNIST Model} The model that we use for image classification on the MNIST dataset is the CNN represented in Figure~\ref{fig:cnn_architecture}. Its precise architecture is described in Table~\ref{tab:mnist_cnn}. The model is trained by minimizing the loss \eqref{eq:classification_train_loss} with $\lambda = 10^{-1}$. We train this model with Adam (default Pytorch hyperparameters) for 10 epochs. Across the different runs, the accuracy of the resulting model on the test set ranges between $94-96 \%$ (note that the weight decay decreases the performances, training the same model with $\lambda = 0$ yields a test accuracy above $99 \%$).

\begin{figure}
  \begin{center}
  \includegraphics[width=\textwidth]{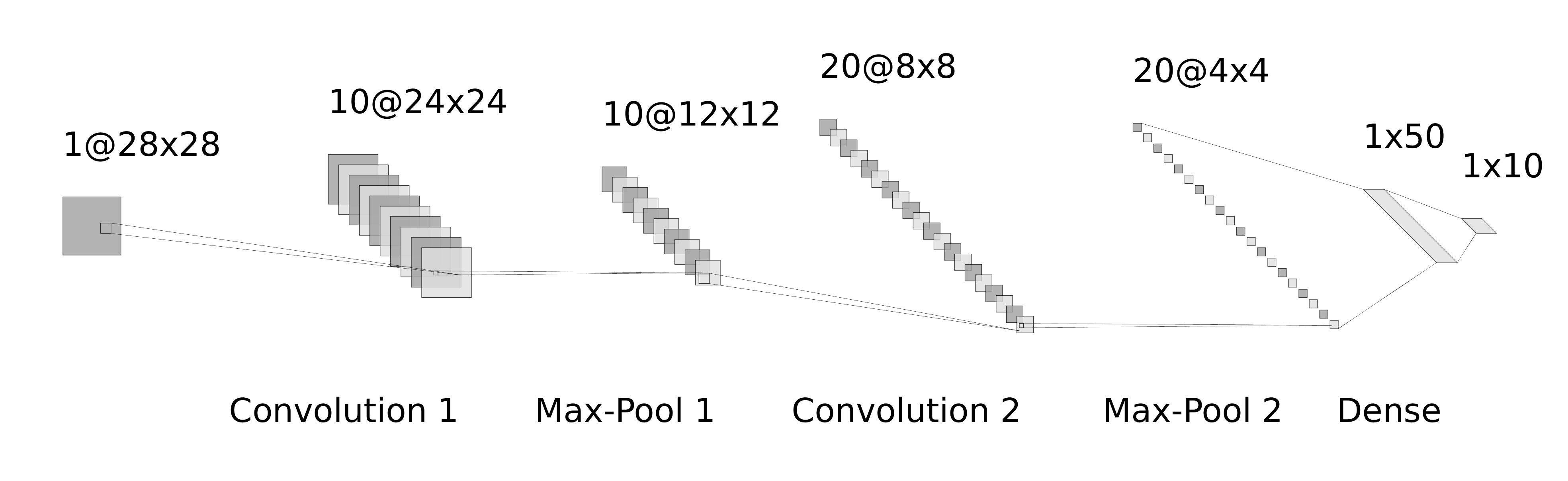}
  \end{center}
  \caption{Architecture of MNIST classifier.}
  \label{fig:cnn_architecture}
\end{figure}

\begin{table}
\begin{tabularx}{\textwidth}{c c c c X}
\toprule
Layer & Input Dimension & Output Dimension & Activation & Remark \\ 
\midrule
Convolution 1 & $28 \times 28 \times 1$ & $24 \times 24 \times 10$ & & Kernel Size: $5$ \\
Max-Pool 1 & $24 \times 24 \times 10$ & $12 \times 12 \times 10$ & ReLU & Kernel Size: $2$ \\
Convolution 2 & $12 \times 12 \times 10$ & $8 \times 8 \times 20$ & & Kernel Size: $5$\\
Dropout & $8 \times 8 \times 20$ & $8 \times 8 \times 20$ & &  \\
Max-Pool 2 & $8 \times 8 \times 20$ & $4 \times 4 \times 20$ & ReLU & Kernel Size: $2$ \\
Flatten & $4 \times 4 \times 20$ & $320$ &  & \\
Dense & $320$ & $50$ & ReLU  & \\
Dropout & $50$ & $50$ &   & Output: $\h = \g(\x)$ \\
Linear & $50$ & $10$ &   & Output: $\y = \l(\h) $ \\
Softmax & $10$ & $10$ &  & Output: $\p = \f(\x)$  \\
\bottomrule
\end{tabularx} 
\vspace{.5cm}
\caption{MNIST Classifier CNN.}
\label{tab:mnist_cnn}
\end{table}

\textbf{Representer theorem} Previous works established that the pre-activation output of classification deep-networks can be decomposed in terms of contributions arising from the training set~\cite{Yeh2018}. In our set-up, where the neural network takes the form $\f = \boldsymbol{\phi} \circ \l \circ \g$, the decomposition can be written as
\begin{align*}
\y &= \l \circ \g (\x) \\
& = - \frac{1}{2 \lambda | \Dtrain |} \sum_{(\x', \z') \in \Dtrain} \partderiv{\mathcal{L}_{\text{train}}}{\left[ \l (\x') \right]} \\
& =  \frac{1}{2 \lambda | \Dtrain |} \sum_{(\x', \z') \in \Dtrain} \left[ \z' - \f(\x') \right] \cdot \left[ \g (\x ')\right]^{\top} \left[ \g (\x) \right],
\end{align*}
where $\lambda$ is the $L^2$ regularization coefficient used in the training loss~\eqref{eq:classification_train_loss}. In our work, we decompose the same output in terms of a corpus $\C$ that is distinct from the training set $\Dtrain$. In this experiment, the corpus is a random subset of the training set $\C \subset \Dtrain$ (in our implementation of the representer theorem, the true labels associated to the corpus example are included). To give a corpus approximation of the output with the representer theorem, we restrict the above sum to the corpus:  
\begin{align*}
\hat{\y} =  \frac{1}{2 \lambda | \C |} \sum_{(\x', \z') \in \C} \left[ \z' - \f(\x') \right] \cdot \left[ \g (\x ')\right]^{\top} \left[ \g (\x) \right].
\end{align*}  
The $R^2_{\Y}$ score reported in the main paper measure the quality of this approximation. It turns out that $R^2_{\Y} < 0$ in both experiments, which indicates that the representer theorem offers poor approximations. We have two explanations: (1)~As previously mentioned, the representer theorem assumes that the decomposition involves the \emph{whole} training set. By making a decomposition that involves a subset $\C$ of the training set, we violate a first assumption of the representer theorem. (2)~The representer theorem assumes that the trained model $\f_{\boldsymbol{\theta}^*}$ corresponds to a stationnary point of the loss: $\nabla_{\boldsymbol{\theta}} \mathcal{L}_{\text{train}} \mid_{\boldsymbol{\theta}^*} = 0$. This assumption is rarely verified in non-convex optimization problems such as the optimization of deep networks.

\textbf{Plots with an alternative metric} In Figures~\ref{fig:prostate_norm} \& \ref{fig:mnist_norm}, we report the norm of the error associated to each method as a function of the number of active corpus members $K$. With this metric, SimplEx remains the most interesting approximation method in latent and output space. Note that for SimplEx, when $K=C$, the error in latent space is equivalent to the corpus residual $r_{\C}(\h)$.

\begin{figure}
\begin{center}
  \begin{subfigure}{.4\textwidth}
  \centering
  \includegraphics[width=\linewidth]{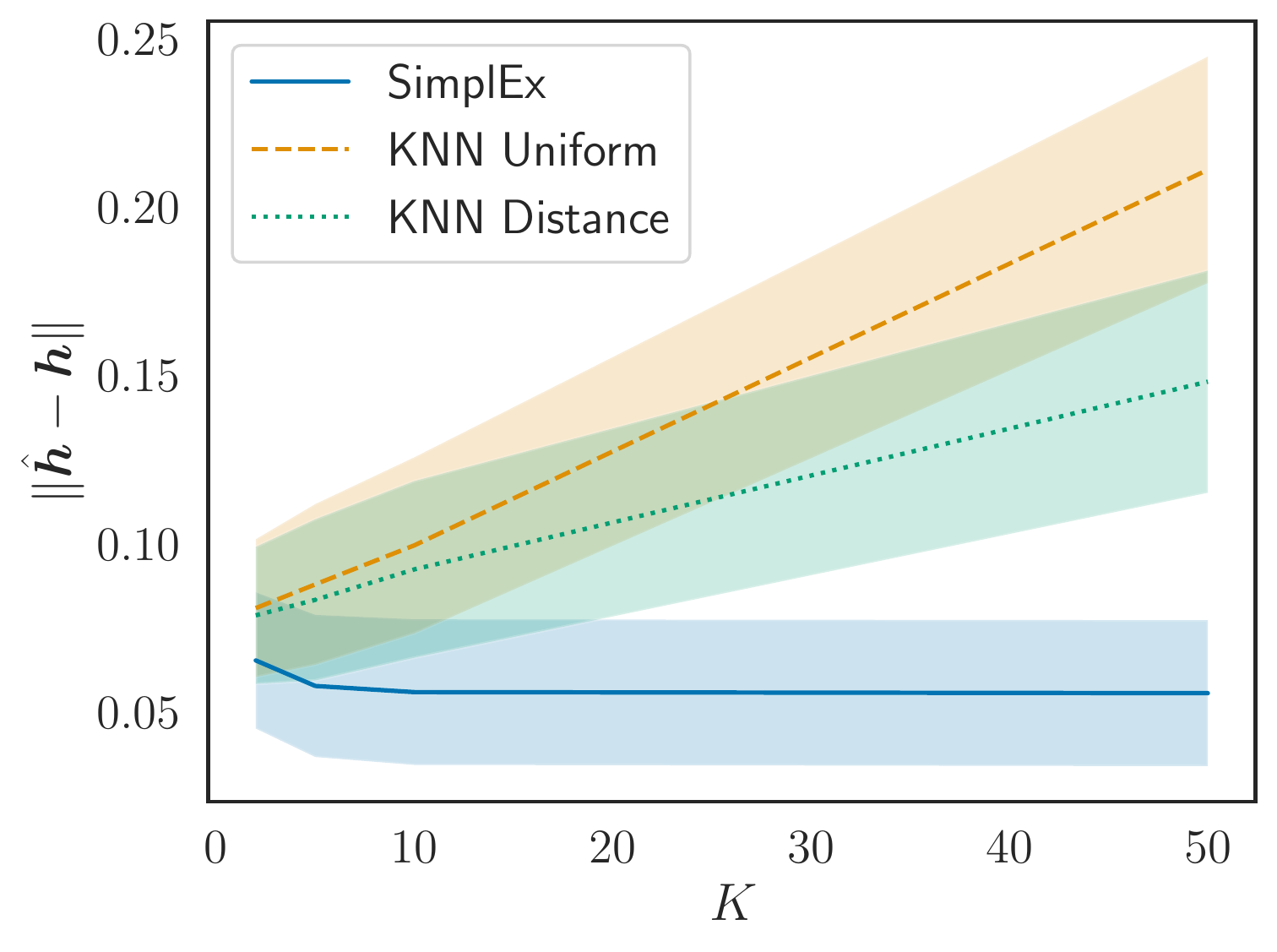}
  \caption{Norm of the latent error}
\end{subfigure}%
\begin{subfigure}{.4\textwidth}
  \centering
  \includegraphics[width=\linewidth]{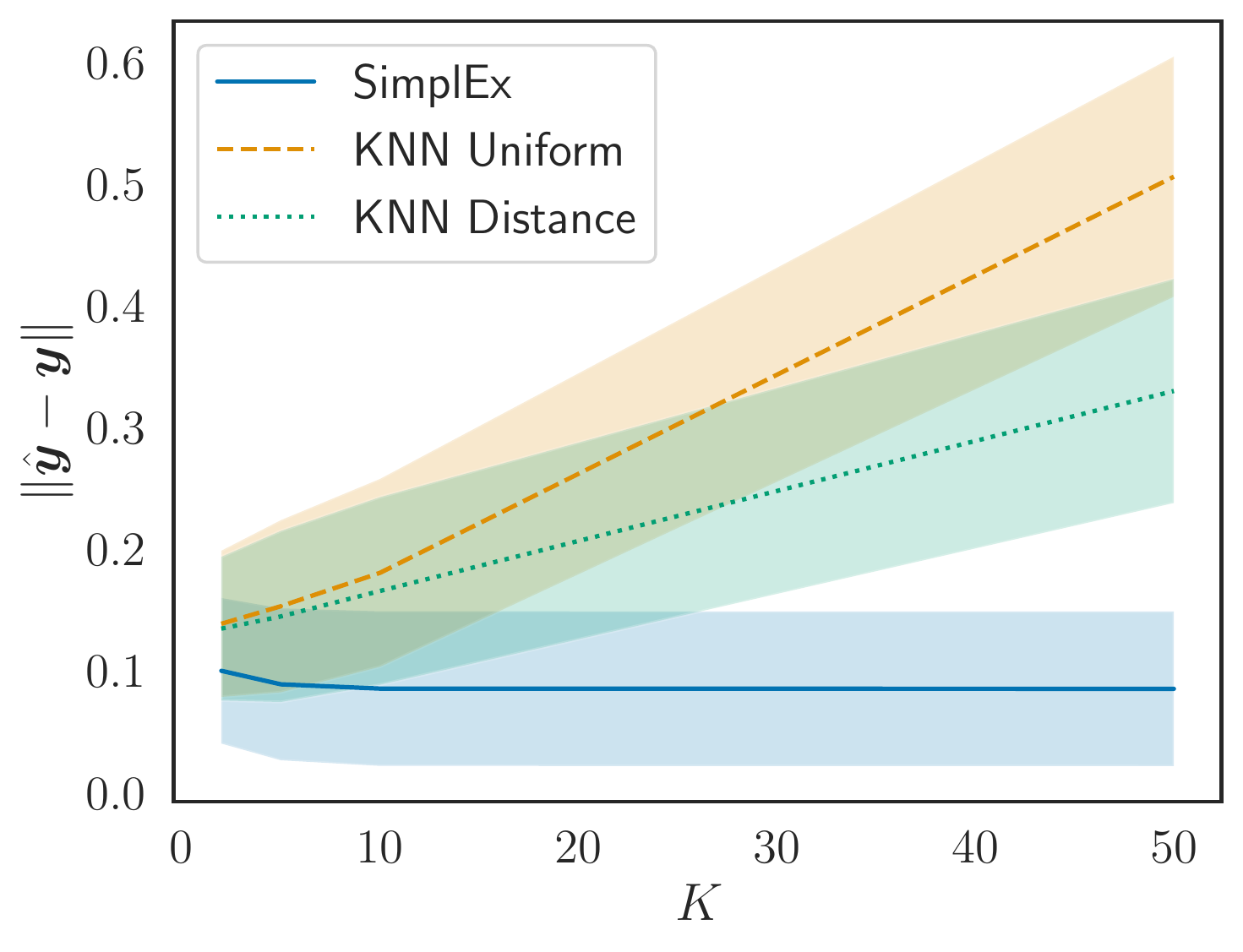}
  \caption{Norm of the output error}
\end{subfigure}
\end{center}
\caption{Precision of corpus decomposition for prostate cancer (avg $\pm$ std).}
\label{fig:prostate_norm}
\end{figure}

\begin{figure}
\begin{center}
  \begin{subfigure}{.4\textwidth}
  \centering
  \includegraphics[width=\linewidth]{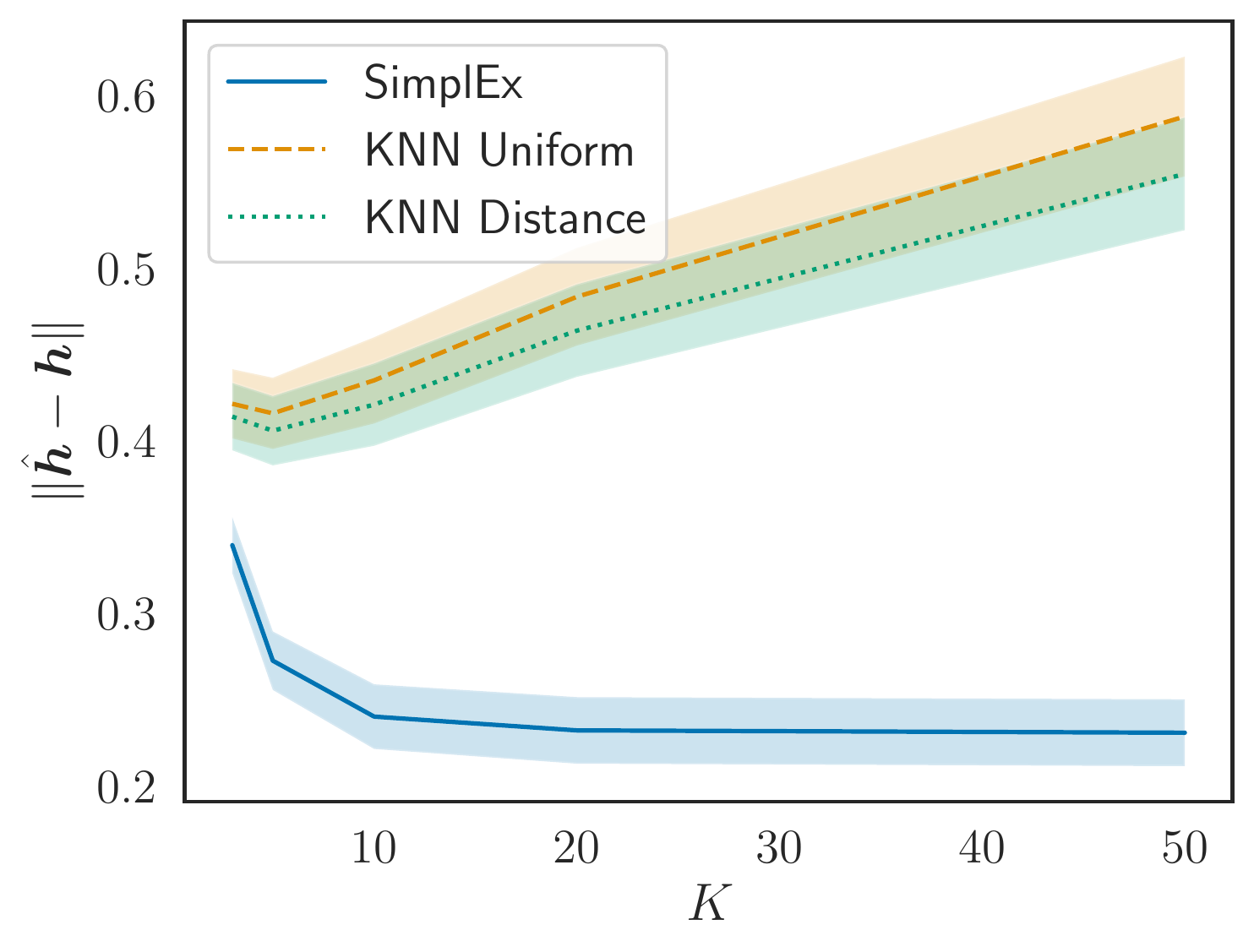}
  \caption{Norm of the latent error}
\end{subfigure}%
\begin{subfigure}{.4\textwidth}
  \centering
  \includegraphics[width=\linewidth]{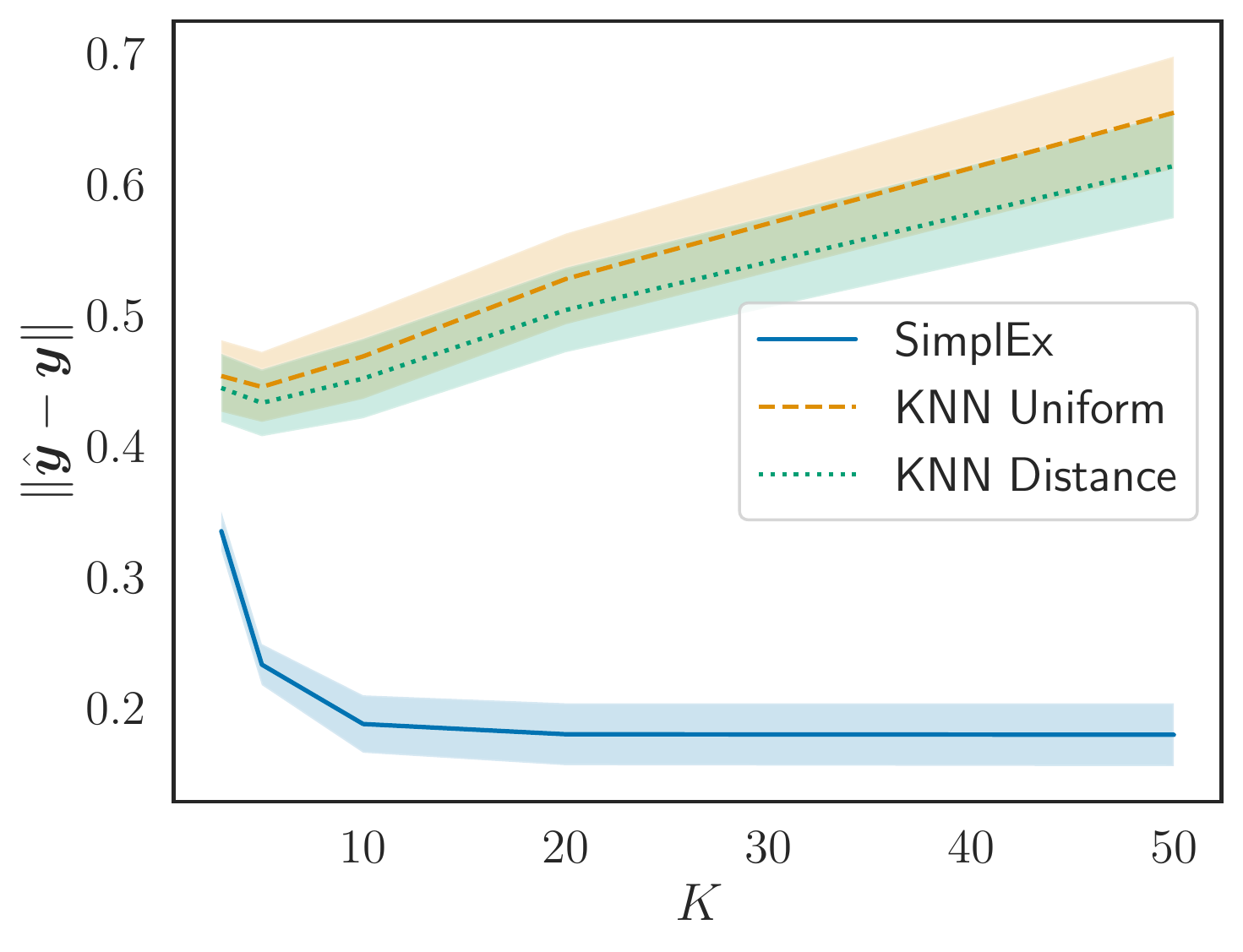}
  \caption{Norm of the output error}
\end{subfigure}
\end{center}
\caption{Precision of corpus decomposition for MNIST (avg $\pm$ std).}
\label{fig:mnist_norm}
\end{figure}

\textbf{On the measure of consistency} In our experiments, we use the standard deviation of each metric across different runs to study the consistency of the corpus approximations. Table~\ref{tab:reset_procedure} details the parts of the experiment that are modified from one run to another. 

\begin{table}
\begin{center}
\begin{tabularx}{.6\textwidth}{X c c c}
\toprule
Modified on each run & \multicolumn{3}{c}{Dataset}\\
       & Prost. Cancer & MNIST & AR \\
\midrule
Data   &  &  & \checkmark \\
Train-Test Split  &\checkmark  &  & \checkmark \\
Model   & \checkmark  & \checkmark  & \checkmark \\
Corpus   & \checkmark  & \checkmark  & \checkmark \\

\bottomrule
\end{tabularx}
\end{center} 
\vspace{.5cm}
\caption{Characterization of a run for each dataset.}
\label{tab:reset_procedure}
\end{table}

\subsection{Details for the clinical use case} \label{subsec:clinical_use_case_sup}

\textbf{Description} In this subsection, we detail the detection of British patients that is described in Section~\ref{subsec:clinical_use_case} of the main paper.

\textbf{Metrics} Each method produces an ordered list $(\x^m)_{m=1}^{\vert \T \vert}$ of elements of $\T$. We inspect its elements in order and count the number of examples that are in $\T \cap \Duk$. This count is represented by the sequence $(u_n)_{n=1}^{\vert \T \vert}$ where $u_n = \vert (\x^m)_{m=1}^{n} \bigcap  \Duk \vert$. A sequence $(u_n)_{n=1}^{\vert \T \vert}$ that increases more quickly is better as it corresponds to a more efficient detection of the British patients.  The maximal baseline corresponds to the upper bound. The experiments are repeated 10 times to report the average metrics together with their standard deviations. 

\textbf{Baselines} We compare SimplEx with 5 indicative baselines. Each method produces an ordered list $(\x^m)_{m=1}^{\vert \T \vert}$ of elements of $\T$. In the case of SimplEx, this list is produced by sorting the examples in decreasing order of corpus residual. We use the two Nearest Neighbours baselines from the previous examples by fixing $K$ to the value that produced the best approximations ($K = 7$). For both of these baselines, we sort the examples in decreasing order of residual $\| \h - \hat{\h}\|$. We consider the random baseline where the order of the list is chosen randomly. Finally, we introduce the ideal baseline that detects all outliers with $\nicefrac{\vert \T \vert}{2} = 100$ inspections: $(\x^m)_{m=1}^{\nicefrac{\vert \T \vert}{2}} = \T \cap \Dout$.

\textbf{CUTRACT Dataset} The CUTRACT dataset is a private dataset consisting in 10,086 patients enrolled in the British Prostate Cancer UK program~\cite{Cutract2019}. All the patients from the CUTRACT dataset have been de-identified. We consider the binary classification task of predicting cancer mortality for patients with prostate cancer. Each patient in the dataset is represented by the couple $(\x , \z)$, where $\x$ contains the patient features and $\z \in \{ 0 , 1 \}^2$ is a vector indicating the patient mortality. The features characterizing the patient are the same as for the SEER dataset. These features are summarized in Table~\ref{tab:cutract_features}. Once again, the full dataset is unbalanced, we then choose $\Duk$ as a balanced subset of 2,000 patients. This dataset is private.

\begin{table}
\begin{tabularx}{\textwidth}{c X}
\toprule
Feature & Range \\ 
\midrule
Age & $64-76$ \\ 
PSA & $8-21$ \\ 
Comorbidities & $0, 1, 2, \geq 3$ \\ 
Treatment & Hormone Therapy (PHT), Radical Therapy - RDx (RT-RDx), \newline Radical Therapy -Sx (RT-Sx), CM \\ 
Grade & $1, 2, 3, 4, 5$ \\ 
Stage & $1, 2, 3, 4$ \\ 
Primary Gleason & $1, 2, 3, 4, 5$ \\ 
Secondary Gleason & $1, 2, 3, 4, 5$ \\ 
\bottomrule
\end{tabularx} 
\vspace{.5cm}
\caption{Features for the CUTRACT Dataset.}
\label{tab:cutract_features}
\end{table}

\textbf{Model} We use the same mortality predictor as in the previous experiment (see Table~\ref{tab:mortality_mlp}). The only difference is that no weight decay is included in the optimization ($\lambda = 0$).

\textbf{Note on the prostate cancer datasets} In the medical literature on prostate cancer~\cite{Gordetsky2016}, the grade of a patient can be deduced from the Gleason scores in the following way:

\begin{align*}
	\text{Gleason}1 +\text{Gleason}2\leq 6\Longrightarrow \text{Grade}=1 \\
	\text{Gleason}1=3\wedge \text{Gleason}2=4 \Longrightarrow \text{Grade} = 2 \\
	\text{Gleason}1=4\wedge \text{Gleason}2=3 \Longrightarrow \text{Grade}=3 \\
	\text{Gleason}1+\text{Gleason}2=8 \Longrightarrow \text{Grade}=4 \\
	\text{Gleason}1+\text{Gleason}2 \geq 9 \Longrightarrow \text{Grade}=5
\end{align*}

We noted that this relationship between the grade and the Gleason score was not always verified among the patients in our two prostate cancer datasets. After discussing with our curator, we understood that the data of some patients was collected by using a different convention. Further, some of the data was missing and has been imputed in a way that does not respect the above rule. Clearly, those details are irrelevant if we use this data to train a model to illustrate the functionalities of SimplEx as it is done in our paper. Nonetheless, it should be stressed that this inconsistency with the medical literature implies that our models are only illustrative and should not be used in a medical context. To avoid any confusion, we have removed the Gleason scores from the Figures in the main paper. For completeness, we have included the Gleason scores in Figures \ref{fig:prostate_further_examples1},\ref{fig:prostate_further_examples2},\ref{fig:user_study}. As we can observe, not all the patients verify the above rule. 

\subsection{Detection of EMNIST letters} \label{subsec:emnist_detection}

\textbf{Description} We propose an analogue of the detection of British patient from Section~\ref{subsec:clinical_use_case} in the image classification setting. We train a CNN with a training set extracted from the MNIST dataset $\Dtrain \subset \Dmnist$. Next, we sample a corpus $\C \subset \Dtrain$ of size $C = 1,000$. We are now interested in investigating if the latent representation of test examples from another similar dataset can be distinguished from latent representations of MNIST examples. To that aim, we use the EMNIST-Letter dataset $\Demnist$, which contains images that are similar to MNIST images. There is one major difference between the two datasets: EMNIST-Letter images represent letters, while MNIST images represent numbers. To evaluate quantitatively if this difference matters for the model representation, we consider a mixed set of test examples $\T$ sampled from both $\Dmnist$ and $\Demnist$: $\T \subset \Dmnist \sqcup \Demnist$. We sample 100 examples from both sources: $\mid \T \cap \Dmnist \vert = \vert \T \cap \Demnist \vert = 100$. For the rest, we follow the same procedure as in the clinical use-case: we approximate the latent representation of each example $\h \in \g (\T)$, compute the associated corpus residual $r_{\C}(\h)$ and sort the examples by decreasing order of residual. 

\textbf{Metrics}  We use the same metrics as in Section~\ref{subsec:clinical_use_case_sup}. 

\textbf{Baselines} We use the same baselines as in Section~\ref{subsec:clinical_use_case_sup}.

\textbf{EMNIST-Letter dataset} EMNIST-Letter contains 145,600 images, each representing a handwritten letter~\cite{Cohen2017}. These images have exactly the same format as MNIST images: $28 \times 28$ pixels with one channel. Ryan Cooper holds the copyright of EMNIST dataset, which is a derivative work from original MNIST datasets. EMNIST dataset is made available under the terms of the MIT license.

\textbf{Model} As in the previous experiment, we train the CNN from Table~\ref{tab:mnist_cnn}. In contrast with the previous experiment, we set the weight decay to zero: $\lambda = 0$. The resulting model has more than $99 \%$ accuracy on the test set.

\textbf{Results} Results are shown in Figure~\ref{fig:mnist_outlier}. This suggests that the difference between the two dataset is encoded in their latent representation. SimplEx and the baselines offer similar performances in this case.

\begin{figure}
  \begin{center}
  \includegraphics[width=.6\textwidth]{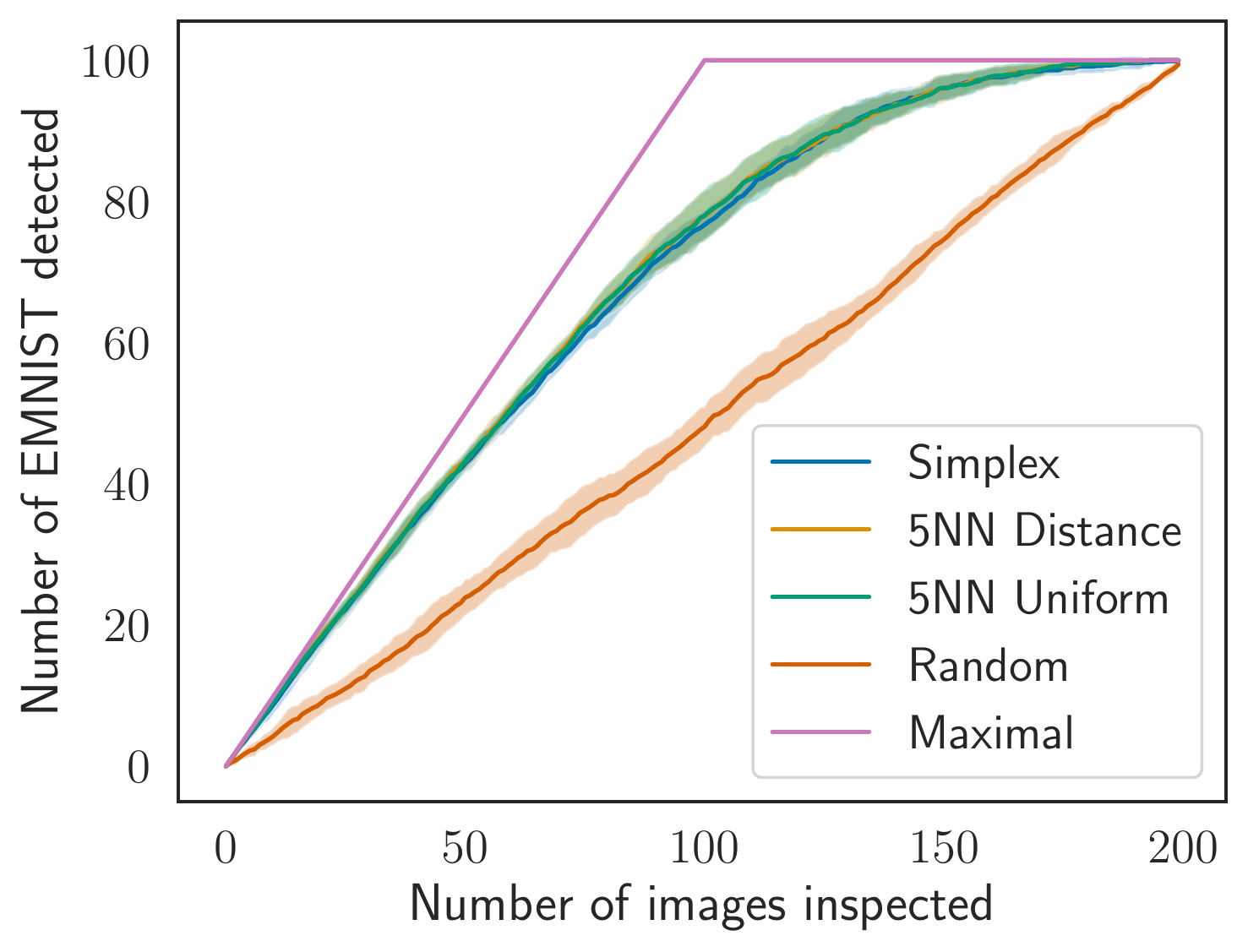}
  \end{center}
  \caption{Detection of EMNIST examples.}
  \label{fig:mnist_outlier}
\end{figure}

\subsection{Experiments with synthetic time series}
In this section, we describe some further experiments we have performed with synthetic time series. The two parts of this subsection mirror the experiments that we have performed in Section~\ref{sec:experiment} of the main paper.
\subsubsection{Corpus precision}
\textbf{Description} This experiment mirrors the experiments from Section~\ref{subsec:precision_experiment} of the main paper. We start with a time series dataset $\D$ that we split into a training set $\Dtrain$ and a testing set $\Dtest$. We train a black-box $f$ for a time series forecasting task on the training set $\Dtrain$. We randomly sample  a set of corpus examples from the training set $\C \subset \Dtrain$ (we omit the true labels for the corpus examples) and a set of test examples from the testing set $\T \subset \Dtest$. For each test example $\x \in \T$, we build an approximation $\hat{\h}$ for $\h = \g (\x)$ with the corpus examples latent representations. In each case, we let the method use only $K$ corpus examples to build the approximation. We repeat the experiment for several values of $K$.    

\textbf{Metrics}  We use the same metrics as in Section~\ref{subsec:precision_experiment} of the main paper. We run the experiment 5 times to report standard deviations across different runs. 

\textbf{Baselines} We use the same baselines as in Section~\ref{subsec:precision_experiment} of the main paper. 

\textbf{Data generation} We generate data from the following AR(2) generating process.
\begin{align} \label{equ:original_ar}
x_t = \varphi_1 \cdot x_{t-1} + \varphi_2 \cdot x_{t-2} + \epsilon_t \hspace{.5cm} \forall t \in [3:T+1],
\end{align}
where $\varphi_1 = .7$, $\varphi_2 = .25$ and $\epsilon_t \sim \mathcal{N}(0, 0.1)$. The initial condition for the time series are sampled independently: $x_1, x_2 \sim \mathcal{N}(0,1)$. Each instance in the dataset $\D$ consists in a couple $(\x , \y) \in \D$ of sequences $\x = (x_t)_{t=1}^T$ and $\y = (y_t)_{t=1}^T$. For each time step, the target indicates the value of the time series at the next step: $y_t = x_{t+1}$ for all $t \in [1:T]$. We generate 10,000 such instances that we split into 9,000 training instances $\Dtrain$ and 1,000 testing instances $\Dtest$.

\textbf{Model} We train a two layer LSTM to forecast the next value of the time series at each time step. The precise model architecture is described in Table~\ref{tab:ar_lstm}. The model is trained by minimizing the following loss:
\begin{align*} 
\mathcal{L}_{\text{train}}\left( \boldsymbol{\theta} \right) = \sum_{\left( \x , \y \right) \in \Dtrain } \sum_{t=1}^T \left[ f_{\boldsymbol{\theta}} (\x_{1:t}) - y_t \right]^2,
\end{align*}
where $\x_{1:t} \equiv (x_t)_{t=1}^T$. We train this model with Adam (default Pytorch hyperparameters) for 20 epochs. Across the different runs, the average RMSE of the resulting model on testing data is always $0.1$. This corresponds to ideal performances for a deterministic model due to the noise term $\epsilon_t$ in the AR model.

\begin{table}
\begin{tabularx}{\textwidth}{c c c c X}
\toprule
Layer & Input Dimension & Output Dimension & Activation & Remark \\ 
\midrule
LSTM 1 & $t \times 1$ &  $t \times 100$ &  & \\
LSTM 2 & $t \times 100$ &  $100$ &   &  Output: $\h_t = \g(\x_{1:t})$ \\
Linear & $100$ & $1$ &   & Output: $y_t = f(\x_{1:t}) $ \\
\bottomrule
\end{tabularx} 
\vspace{.5cm}
\caption{AR Forecasting LSTM, $t$ denotes the length of the input sequence. }
\label{tab:ar_lstm}
\end{table}

\textbf{Results} The results of this experiment are shown in Figure~\ref{fig:ar_r2} \& \ref{fig:ar_norm}. As in Section~\ref{subsec:precision_experiment} of the main paper, SimplEx offers significantly better and more consistent results across different runs.

\begin{figure}
\begin{center}
  \begin{subfigure}{.4\textwidth}
  \centering
  \includegraphics[width=\linewidth]{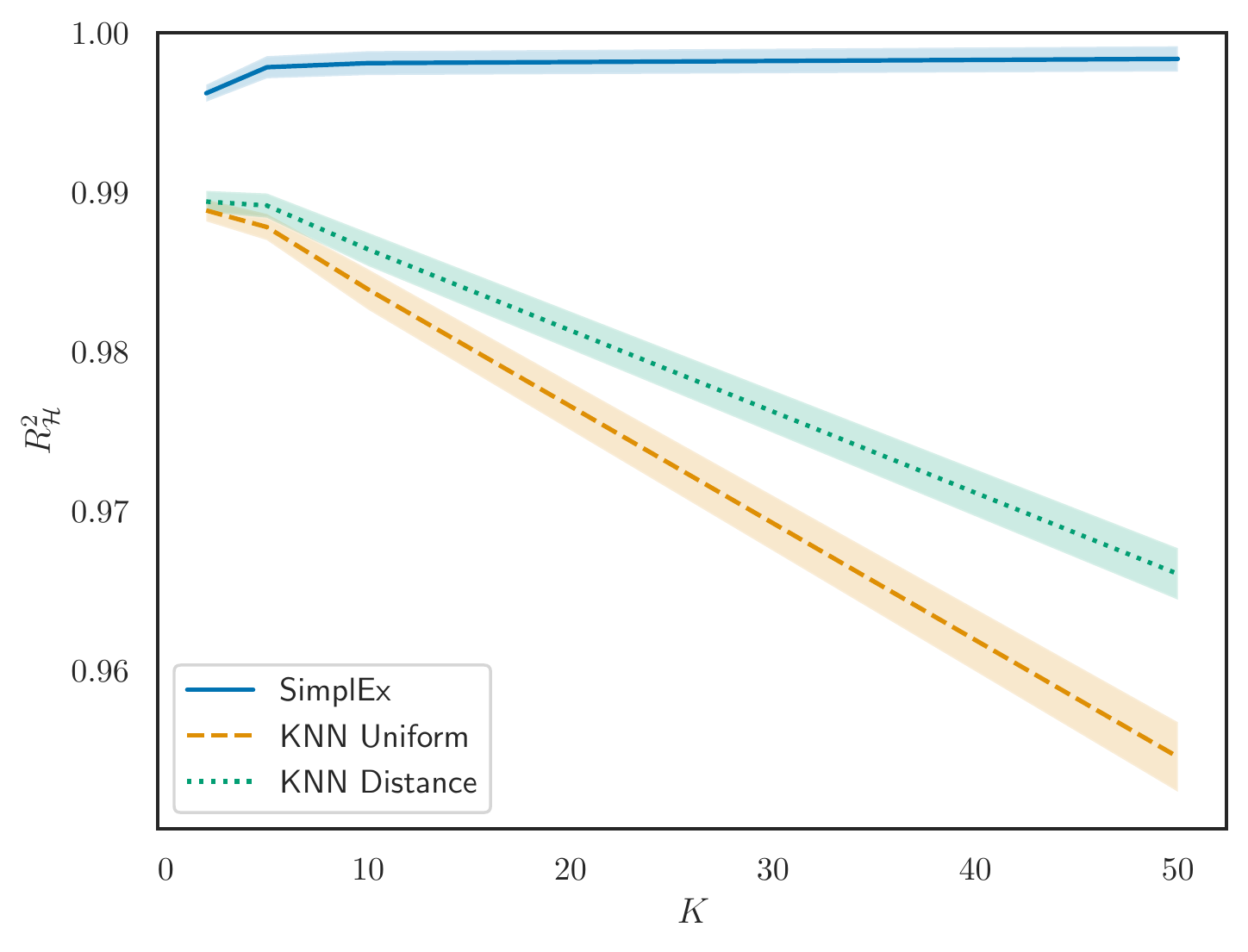}
  \caption{Latent $R^2$ score}
\end{subfigure}%
\begin{subfigure}{.4\textwidth}
  \centering
  \includegraphics[width=\linewidth]{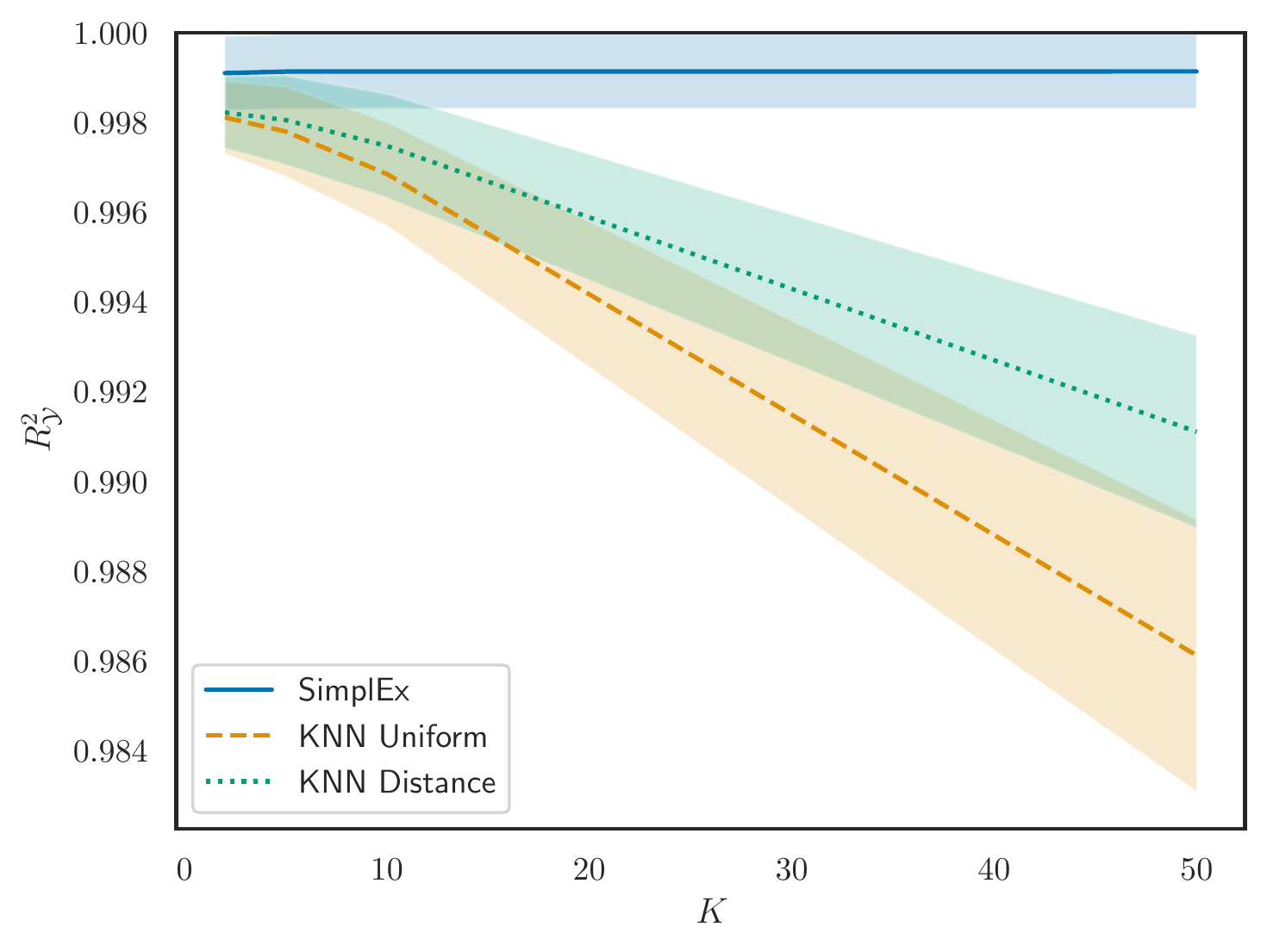}
  \caption{Output $R^2$ score}
\end{subfigure}
\end{center}
\caption{Precision of corpus decomposition for AR (avg $\pm$ std).}
\label{fig:ar_r2}
\end{figure}

\begin{figure}
\begin{center}
  \begin{subfigure}{.4\textwidth}
  \centering
  \includegraphics[width=\linewidth]{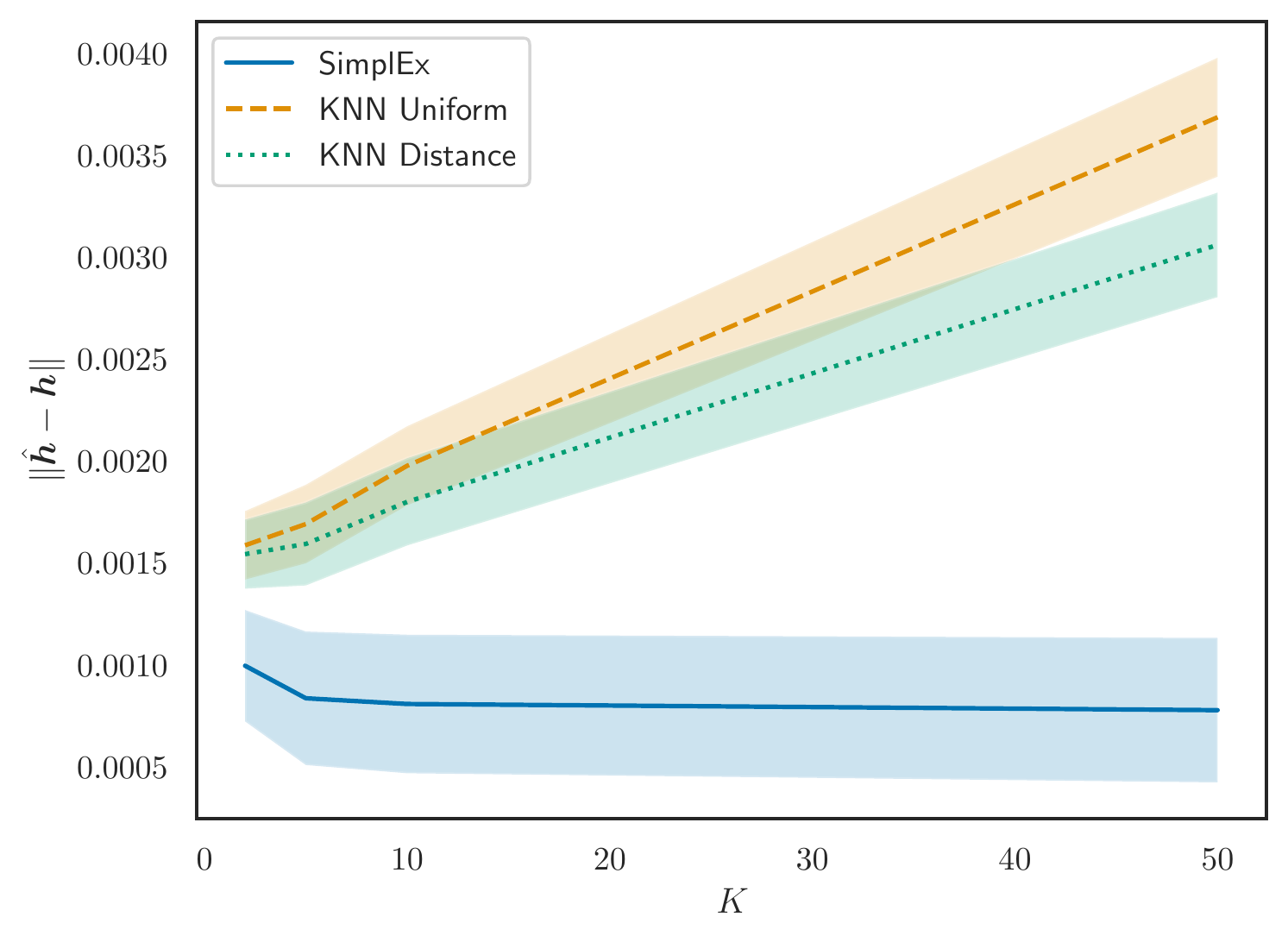}
  \caption{Norm of the latent error}
\end{subfigure}%
\begin{subfigure}{.4\textwidth}
  \centering
  \includegraphics[width=\linewidth]{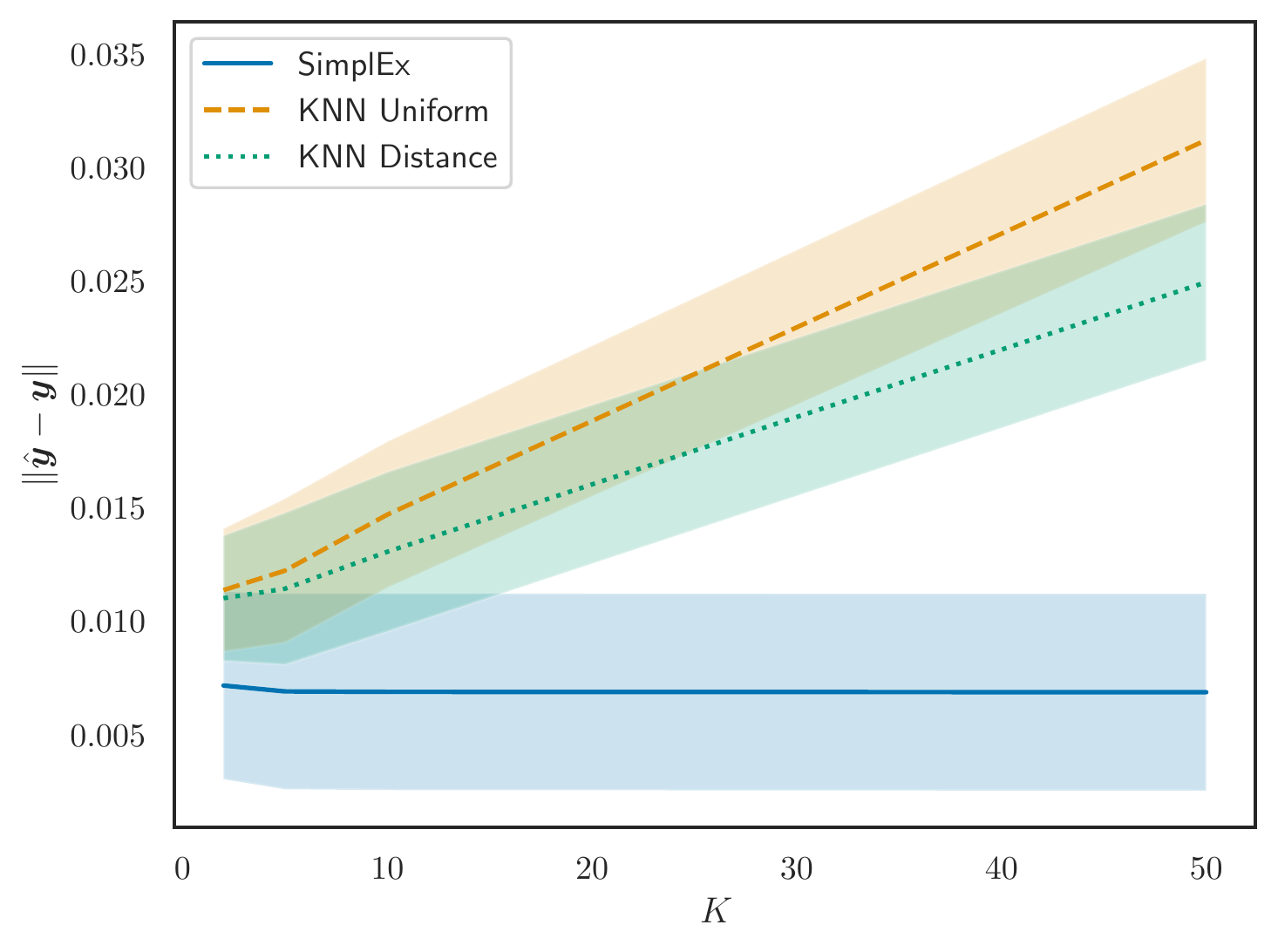}
  \caption{Norm of the output error}
\end{subfigure}
\end{center}
\caption{Precision of corpus decomposition for AR (avg $\pm$ std).}
\label{fig:ar_norm}
\end{figure}

\subsubsection{Detection of oscillating time series}

\textbf{Description} This experiment mirrors the EMNIST detection experiment from Section~\ref{subsec:emnist_detection}. We train a LSTM with a training set extracted from the previous AR dataset $\Dtrain \subset \D$. Next, we sample a corpus $\C \subset \Dtrain$ of size $C = 1,000$. We are now interested in investigating if the latent representation of test examples from another similar dataset can be distinguished from latent representations of traditional AR examples. To that aim, we use a dataset sampled from a distinct AR(2) process $\tilde{\D}$. To evaluate quantitatively if this difference matters for the model representation, we consider a mixed set of test examples $\T$ sampled from both $\D$ and $\tilde{\D}$: $\T \subset \Dtest \sqcup \tilde{\D}$. We sample 1,000 examples from both sources: $\mid \T \cap \Dtest \vert = \vert \T \cap \tilde{\D} \vert = 1,000$. For the rest, we follow the same procedure as in the clinical use-case: we approximate the latent representation of each example $\h \in \g (\T)$, compute the associated corpus residual $r_{\C}(\h)$ and sort the examples by decreasing order of residual.  

\textbf{Metrics}  We use the same metrics as in Section~\ref{subsec:clinical_use_case_sup}. We run the experiment 5 times to report standard deviations across different runs. 

\textbf{Baselines} We use the same baselines as in Section~\ref{subsec:clinical_use_case_sup}. 

\textbf{Data generation} We generate $\D$ as in the previous experiment. The time series in $\tilde{\D}$ are sampled from the following AR(2) process:
\begin{align} \label{equ:oscillating_ar}
\tilde{x}_t = - \varphi_1 \cdot \tilde{x}_{t-1} + \varphi_2 \cdot \tilde{x}_{t-2} + \epsilon_t \hspace{.5cm} \forall t \in [3:T+1],
\end{align}
where $\varphi_1$, $\varphi_2$ and $\epsilon_t$ are defined as in~\eqref{equ:original_ar}. The initial condition for the time series are sampled independently: $\tilde{x}_1, \tilde{x}_2 \sim \mathcal{N}(0,1)$. The only difference between $\D$ and $\tilde{\D}$ lies in the extra minus sign from \eqref{equ:oscillating_ar} compared to \eqref{equ:original_ar}. This gives an extra oscillating behaviour to the time series from $\tilde{\D}$. We generate 1,000 such instances that we use for testing purpose.

\textbf{Model} We use the same LSTM as in the previous experiment.

\textbf{Results} The results of this experiment are shown in Figure~\ref{fig:ar_outlier}. As in Section~\ref{subsec:clinical_use_case} of the main paper, the difference between $\D$ and $\tilde{\D}$ is imprinted in their latent representation. Once again, SimplEx offers the best detection scheme. 

\begin{figure}[h]
  \begin{center}
  \includegraphics[width=.6\textwidth]{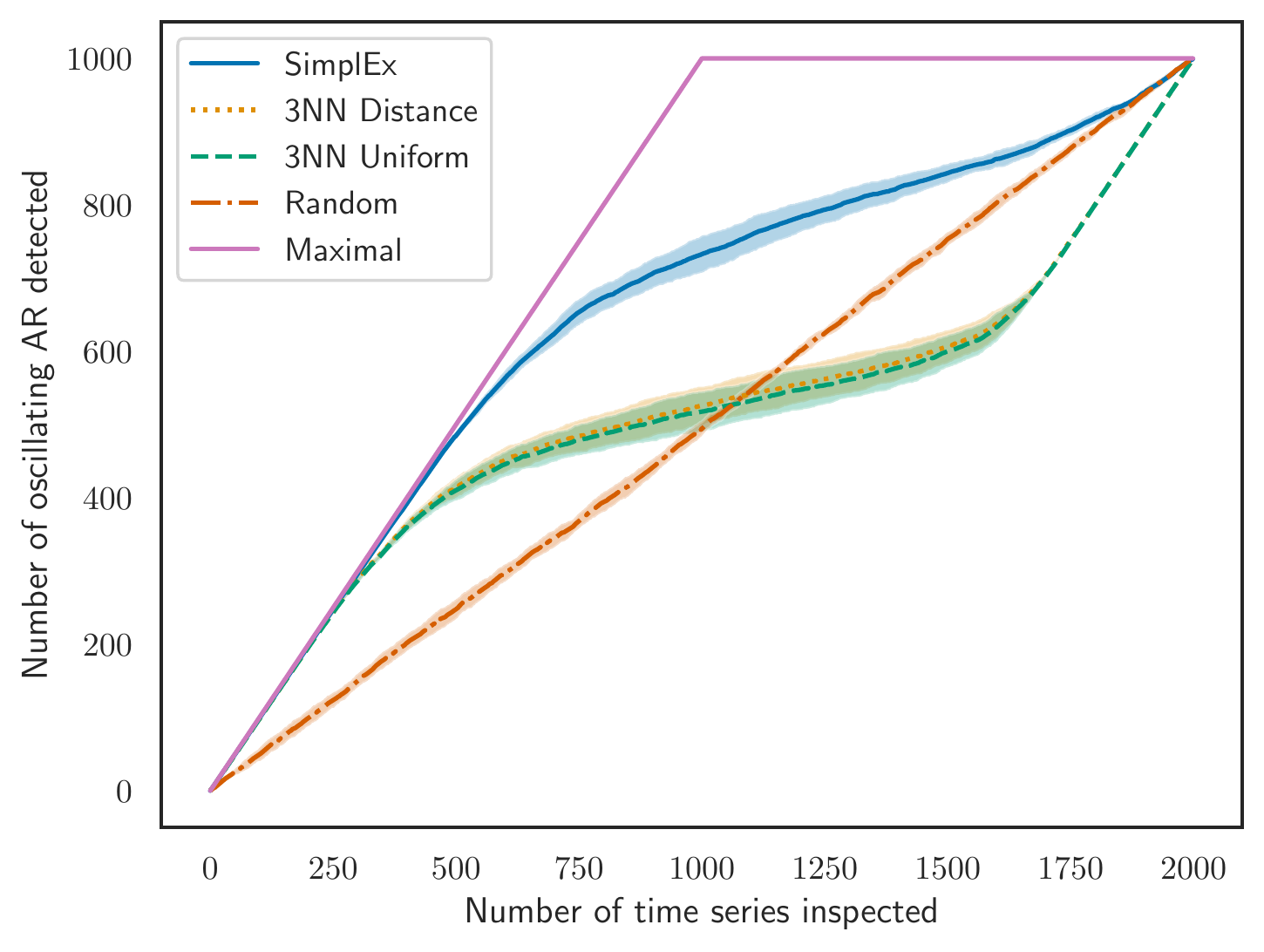}
  \end{center}
  \caption{Detection of oscillating AR.}
  \label{fig:ar_outlier}
\end{figure}

\subsection{A comparison with Influence Functions}

\textbf{Description }We note that there is no standard way to reconstruct the explicit black-box output $\f(\x)$ with the influence scores~\cite{Koh2017} for an input $\x \in \X \subset \R^{d_X}$. In contrast, SimplEx allows to explicitly decompose a black-box prediction in terms of contributions arising from each corpus example: $\f(\x) = \sum_{c=1}^C w^c \l (\h^c)$. An interesting question to ask is the following: can we interpret influence scores as reconstruction weights in latent space? To explore this question, we propose the following procedure. First, we compute the influence score $i^c \in \mathbb{R}$ for the prediction $\textbf{f}(\textbf{x})$ and for each corpus example $\textbf{x}^c \in \mathcal{C}$. Then, we extract the helpful examples from the corpus:  $\mathcal{C}_{help} = \{ x^c \in \mathcal{C} \mid i^c > 0 \}$. In the same spirit as in Section~\ref{subsec:precision_experiment} of the main paper, we select the $K$ most helpful examples from $\mathcal{C}_{help}$. We denote their corpus indices as $\mathcal{I} = \{ c_1, c_2, \dots, c_K \} \subset [C]$.
Finally, we make a corpus decomposition with weights proportional to the influence score:
\begin{align}
	w^c =  \left\{ \begin{array}{ll}
		\frac{i^c}{\sum_{k \in \mathcal{I}}i^{k}} & \text{if } c \in \mathcal{I} \\
		0 & \text{else}
	\end{array}\right.
\end{align}

\textbf{Metrics} We study the quality of influence-based corpus decomposition $\sum_{c=1}^C w^c \textbf{h}^c$ as an approximation of the test example's latent representation $\textbf{h} = \textbf{g}(\textbf{x})$. Therefore, we use the same metrics as in Section~\ref{subsec:precision_experiment} of the main paper.

\textbf{Baseline} We consider SimplEx as a baseline.

\textbf{Dataset} We perform the experiment with the MNIST dataset.

\textbf{Results} We report the result of this experiment in Figure~\ref{fig:influence} (average +/- standard deviation over 5 runs). This confirms that influence functions scores are not suitable to decompose the latent representations in terms of the corpus. 

\begin{figure}
	\begin{center}
		\includegraphics[width=.8\textwidth]{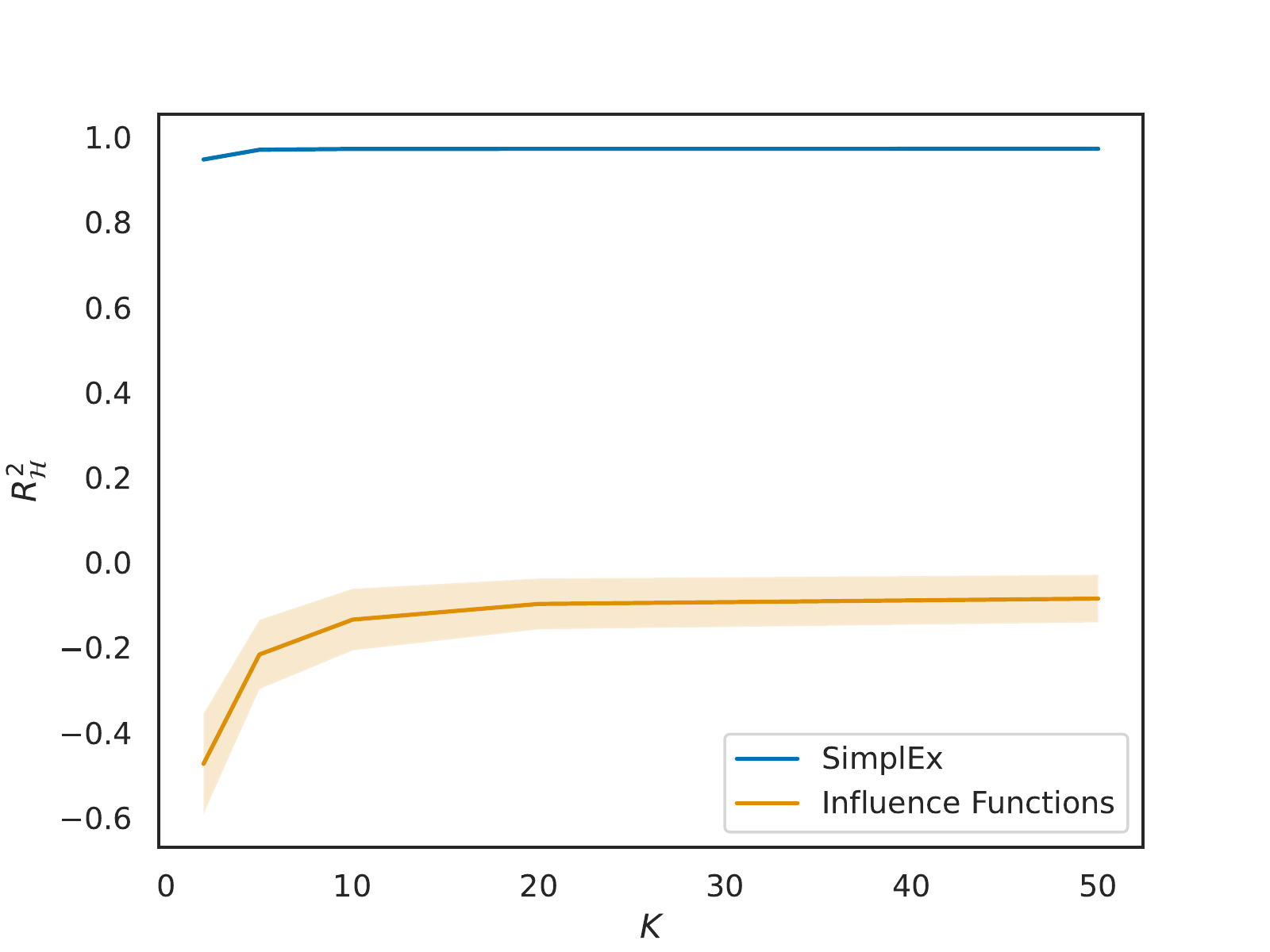}
	\end{center}
	\caption{Precision of the corpus decomposition in latent space.}
	\label{fig:influence}
\end{figure}

\subsection{More examples}

In Figures~\ref{fig:mnist_further_examples1}-\ref{fig:prostate_further_examples2}, we provide further examples of corpus decompositions with MNIST and SEER. 

\begin{figure}
  \begin{center}
  \includegraphics[width=\textwidth]{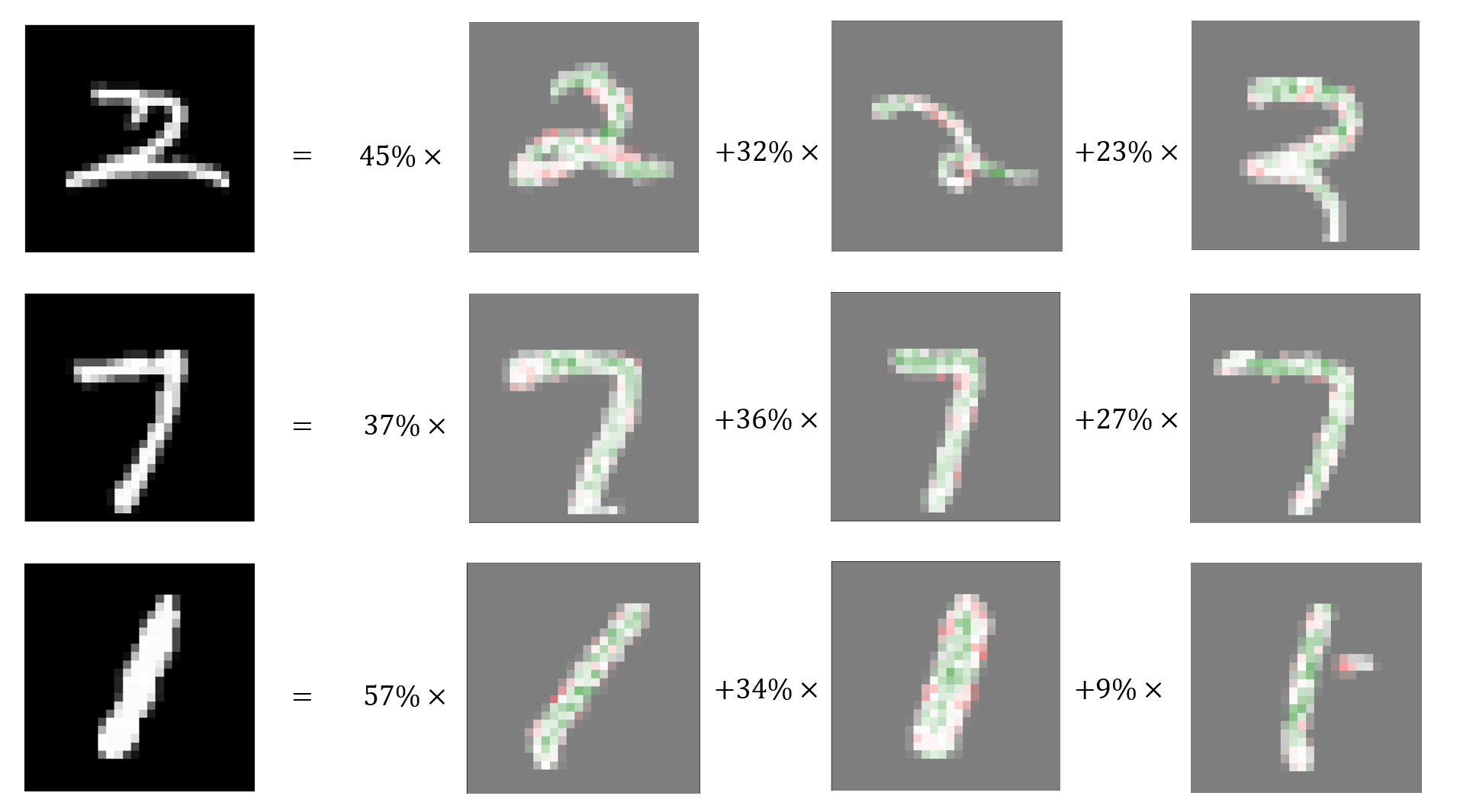}
  \end{center}
  \caption{Examples of MNIST decompositions (left: test example, right: corpus decomposition).}
  \label{fig:mnist_further_examples1}
\end{figure}

\begin{figure}
  \begin{center}
  \includegraphics[width=.8\textwidth]{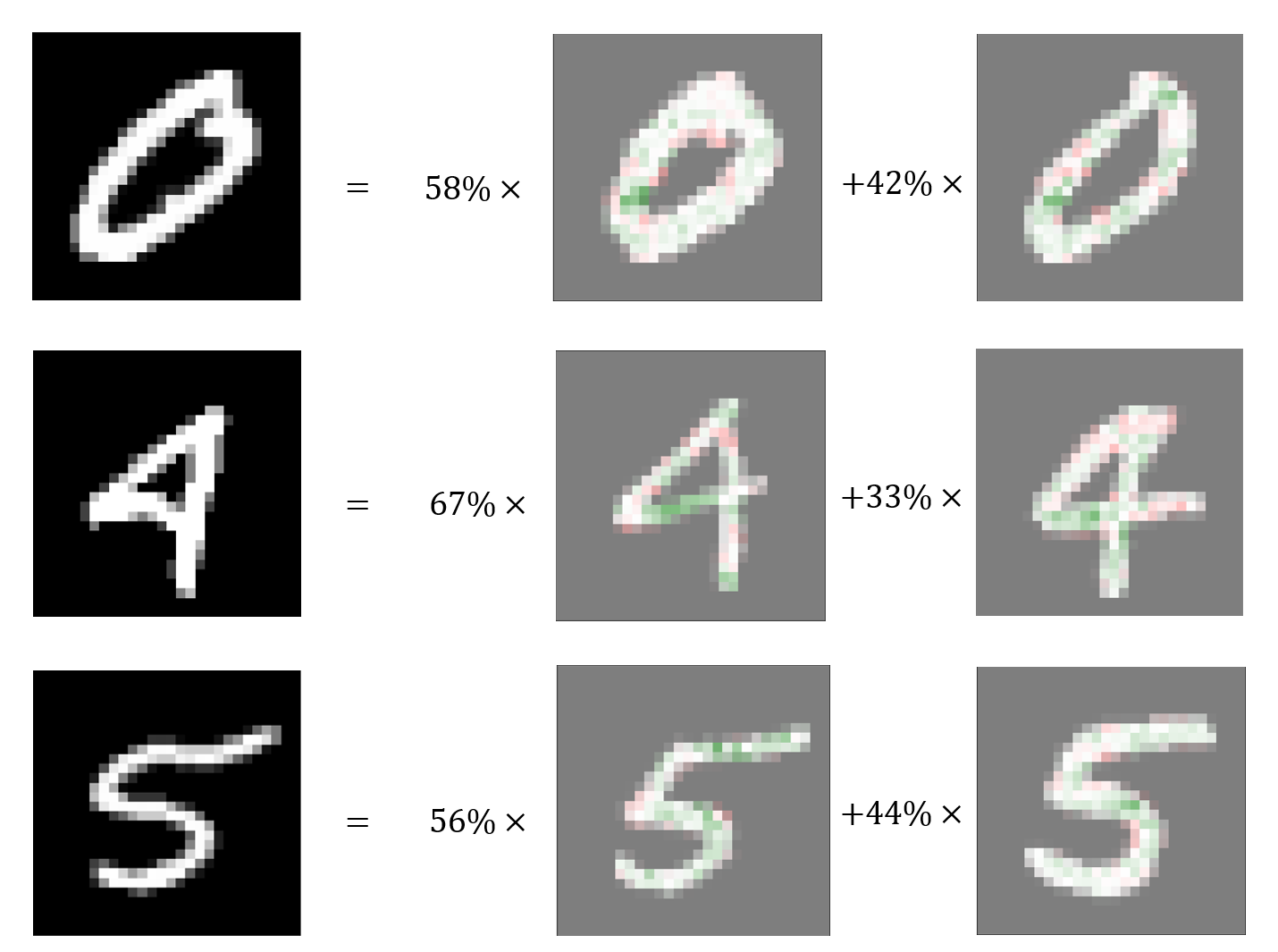}
  \end{center}
  \caption{Examples of MNIST decompositions (left: test example, right: corpus decomposition).}
  \label{fig:mnist_further_examples2}
\end{figure}

\begin{figure}
  \begin{center}
  \includegraphics[width=\textwidth]{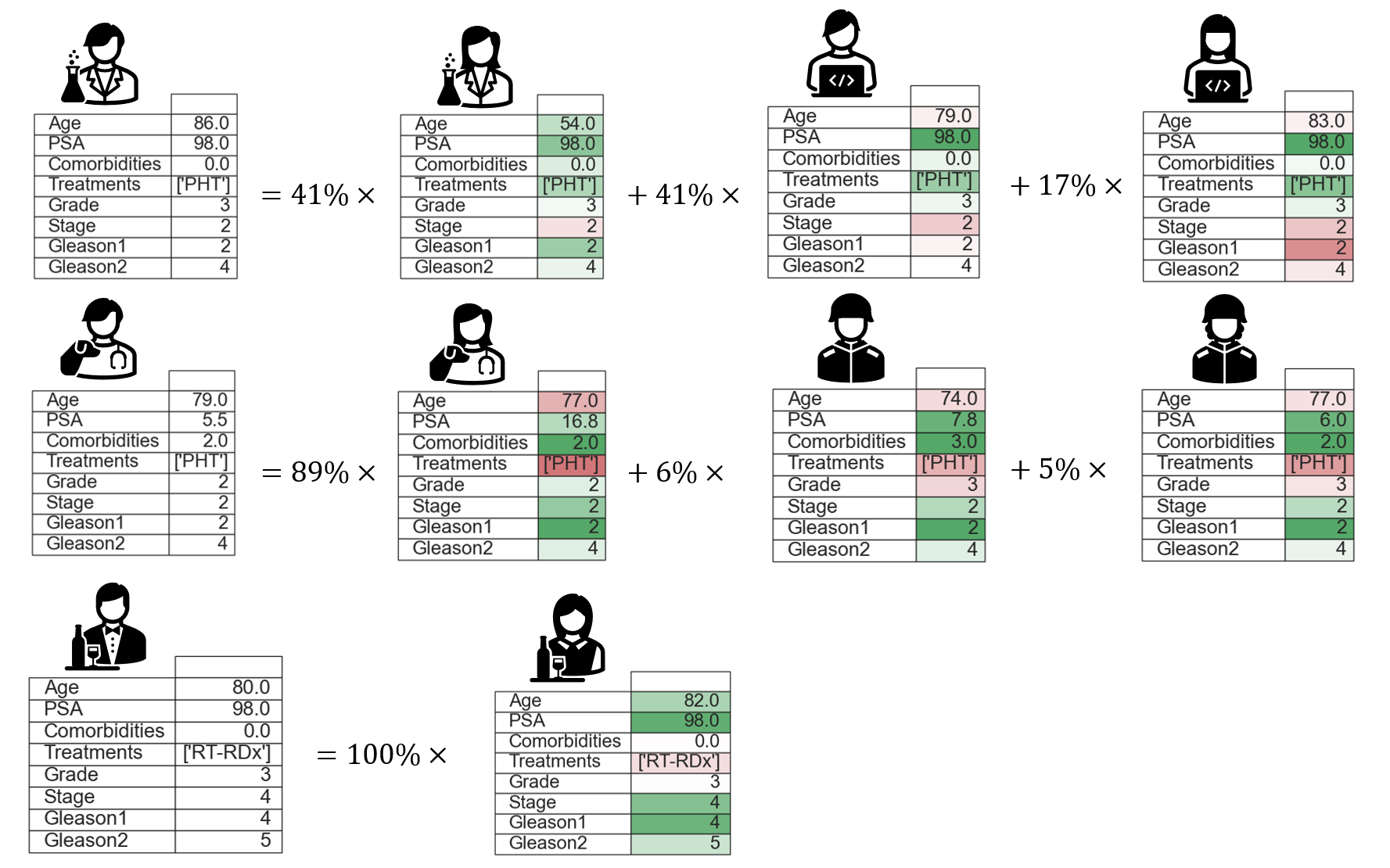}
  \end{center}
  \caption{Examples of SEER decompositions (left: test example, right: corpus decomposition).}
  \label{fig:prostate_further_examples1}
\end{figure}
  
\begin{figure}
  \begin{center}
  \includegraphics[width=.8\textwidth]{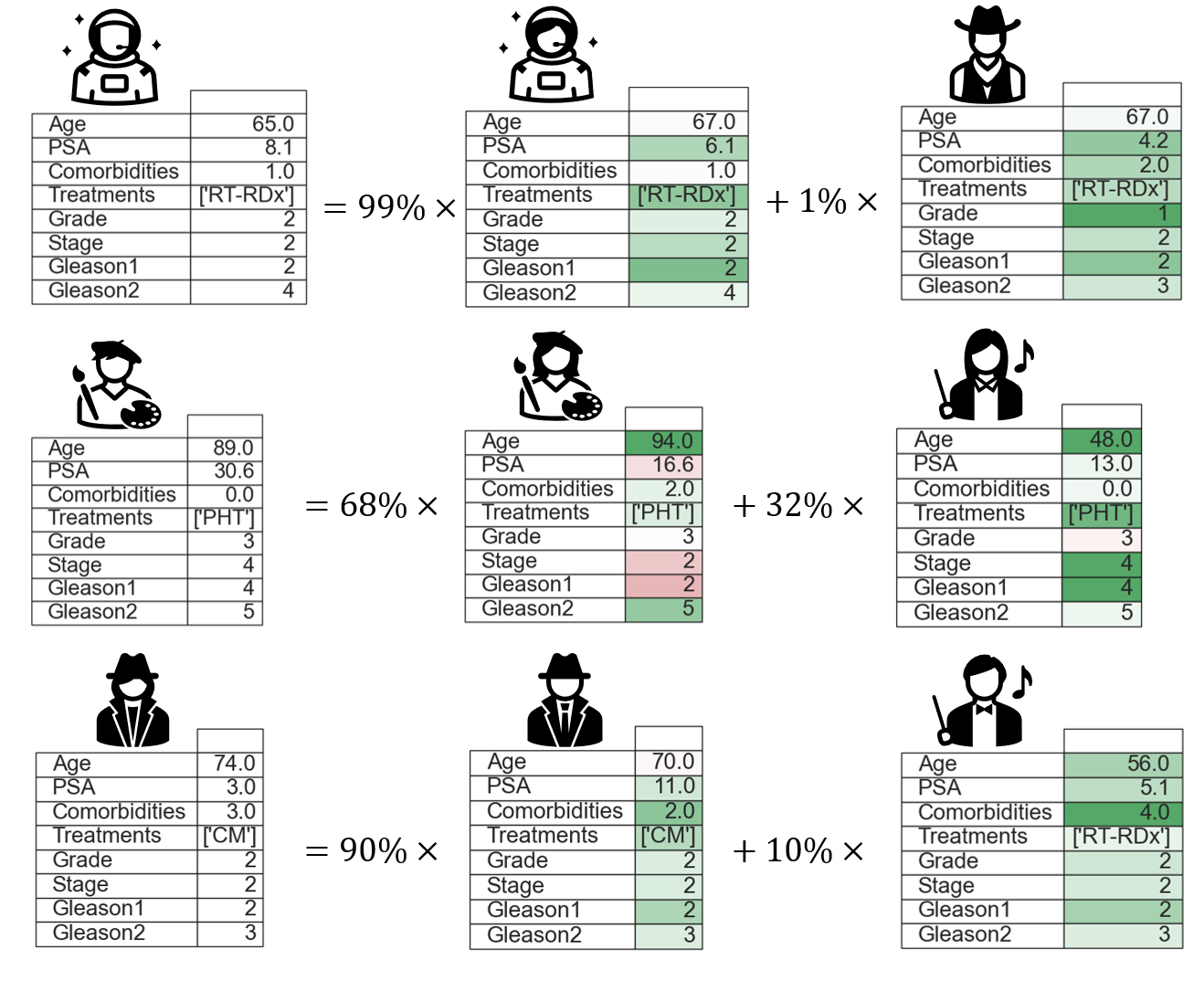}
  \end{center}
  \caption{Examples of SEER decompositions (left: test example, right: corpus decomposition).}
  \label{fig:prostate_further_examples2}
\end{figure}

\FloatBarrier

\section{User Study} \label{sec:user_study}
We have conducted a small scale user study with SimplEx. The purpose of this study was to identify if the functionalities introduced by SimplEx are interesting for the clinicians. In total, 10 clinicians took part in the study.

\begin{figure}
	\begin{center}
		\includegraphics[width=\textwidth]{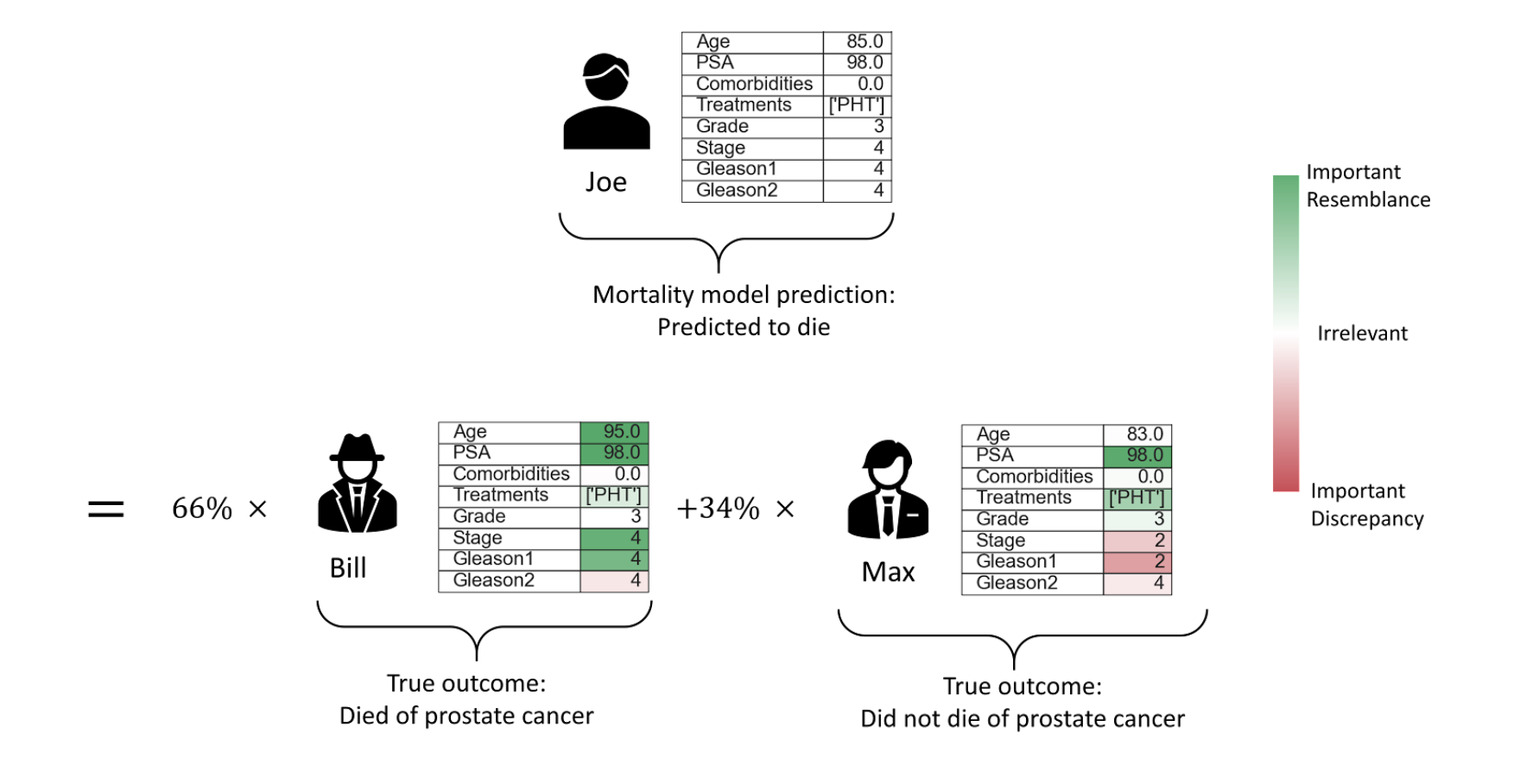}
	\end{center}
	\caption{SimplEx example provided in the user study.}
	\label{fig:user_study}
\end{figure}

Let us now describe the study. With the SEER Prostate Cancer dataset that is described in the paper, we have performed a SimplEx corpus decomposition presented in Figure~\ref{fig:user_study}. The decomposition involved 1 test patient that we called Joe and 2 corpus patients that we called Bill and Max. Our classification model predicted that Joe will die of his prostate cancer. Bill died of his prostate cancer and Max survived. The SimplEx corpus weights were as follows: 66\% for Bill, 34\% for Max. For both Bill and Max, the Jacobian Projections were given and presented as a measure of importance for each of their features in order to relate them to Joe.

After presenting this explanation to the clinician, we gradually brought their attention to its various components. We made several statements related to SimplEx's functionalities and asked the clinicians if they agree/disagree on a scale from 0 to 5, where 0 corresponds to strongly disagreeing, 3 corresponds to a neutral opinion and 5 corresponds to strongly agreeing.

The first two statements were related to the weights appearing in the corpus decomposition. The purpose was to determine if those are important for the clinicians and if there is an additional value in learning these weights, as is done in SimplEx. The first statement was the following: “The value of the weights in the corpus decomposition is important”. The results were the following: 6 of the clinicians agreed (1 strongly), 1 remained neutral and 3 disagreed (1 strongly). The second statement was the following: “Some valuable information is lost in setting the weights to a uniform value” (the doctors are given the KNN Uniform equivalent of SimplEx's explanation in the presented case). The results were the following: 5 of the clinicians agreed (3 strongly), 3 remained neutral, 2 strongly disagreed. We conclude that the majority of the clinicians found the weights to be important. Most of them found that hard-coding the weights as in the KNN Uniform baseline hides some valuable information.

The third statement was related to the Jacobian Projections. The purpose was to determine if the Jacobian Projections provide valuable information for interpretability. The statement was the following: “Knowing which feature increases the similarity/discrepancy between two patients is important”. The results were the following: 9 of the clinicians agreed (5 strongly), 1 disagreed. We conclude that the Jacobian Projections constitute a crucial part of SimplEx's explanations.

The fourth statement was related to the freedom of choosing the corpus. The purpose was to determine if the flexibility of SimplEx is useful in practice. The statement was the following: “It is important for the clinician to be able to choose the patients in the corpus that is used for comparison”. The results were the following: 4 of the clinicians agreed (1 strongly), 1 remained neutral, 5 disagreed (3 strongly). Clearly, the clinicians are more divided on this point. However, this additional freedom offered by SimplEx comes at no cost. A clinician that desires explanations in terms of patients they are familiar with can use their own corpus. A clinician that is happy with explanations in terms of any patients can use a corpus sampled from training data.

The last statement was related to the use of SimplEx in order to anticipate misclassification, as it is suggested in Section 3.2 of the main paper. The statement was the following: “If Bill had not died due to his prostate cancer, this would cast doubt on the mortality predicted for Joe”. The results were the following: 6 of the clinicians agreed (2 strongly), 1 remained neutral, 3 disagreed (2 strongly). We conclude that, for the majority of the clinicians, SimplEx's explanations affect their confidence in the model's prediction. 
\end{document}